\documentclass[10pt,twocolumn,letterpaper]{article}

\usepackage{cvpr}
\def\confName{CVPR}
\def\confYear{2027}

\usepackage{times}
\usepackage{epsfig}
\usepackage{graphicx}
\usepackage{amsmath}
\usepackage{amssymb}
\usepackage{amsthm}
\usepackage{bm}
\usepackage{booktabs}
\usepackage{multirow}
\usepackage{array}
\usepackage{xcolor}
\usepackage{colortbl}
\usepackage{subcaption}
\usepackage{algorithm}
\usepackage{algpseudocode}
\usepackage{float}
\usepackage{placeins}
\usepackage{url}
\definecolor{cvprblue}{rgb}{0.21,0.49,0.74}
\usepackage[pagebackref,breaklinks,colorlinks,allcolors=cvprblue]{hyperref}
\usepackage{flafter}

\newtheorem{proposition}{Proposition}

\definecolor{FADRMBlue}{RGB}{226,239,249}
\definecolor{PSMPurple}{RGB}{241,233,248}

\newcommand{\tbscore}[2]{\ensuremath{#1_{\pm #2}}}
\newcommand{\psmup}[1]{\,{\scriptsize\ensuremath{(\uparrow #1)}}}
\newcommand{\psmdown}[1]{\,{\scriptsize\ensuremath{(\downarrow #1)}}}
\newcommand{\psmeq}[1]{\,{\scriptsize\ensuremath{(= #1)}}}

\title{Difficulty-Aware Statistical Matching for Dataset Distillation}
\author{
  Hongxu Ma$^{1,*}$ \quad
  Guang Li$^{2,*,\dagger}$ \quad
  Shijie Wang$^{3}$ \quad
  Dongzhan Zhou$^{4}$ \quad
  Suorong Yang$^{5}$ \\
  Baoli Sun$^{6}$ \quad
  Takahiro Ogawa$^{2}$ \quad
  Miki Haseyama$^{2}$ \quad
  Zhihui Wang$^{6,\dagger}$ \\[0.5em]
  $^{1}$Zhejiang University \quad $^{2}$Hokkaido University \\
  $^{3}$The University of Queensland \quad $^{4}$Shanghai Artificial Intelligence Laboratory \\
  $^{5}$National University of Singapore \quad $^{6}$Dalian University of Technology \\[0.3em]
  $^{*}$Equal contribution. \quad $^{\dagger}$Corresponding authors. \\
  guang@lmd.ist.hokudai.ac.jp \quad zhihuiwang@dlut.edu.cn
}

\begin{document}

\maketitle
\flushbottom

\begin{abstract}
Dataset distillation (DD) condenses a large original set into a small distilled set with high training utility. Decoupled statistical matching methods substantially reduce distillation time and memory overhead while achieving strong performance. However, they typically supervise all distilled data with running statistics estimated from the entire original set. These statistics mainly capture the average feature distribution while ignoring the difficulty differences among original data, making them insufficient for characterizing the difficulty structure of the original data. To address this issue, we propose \textbf{D}ifficulty-aware \textbf{S}tatistical \textbf{M}atching (DSM) for Dataset Distillation. After pretraining, DSM uses the Global Difficulty Score (GDS) to evaluate image difficulty, sorts the original data according to the scores, and then partitions it into IPC (Images per Class of distilled data) difficulty groups. During distillation, Statistics Updated Again (SUA) updates the teacher’s BN running statistics using original data from each group via forward passes, providing precise supervision with corresponding difficulty for distilled data at each $id_{\mathrm{IPC}}$. Meanwhile, Initial Sample Screening (ISS) initializes distilled data with original images from the corresponding group, providing an effective starting point for precise matching. Experiments across multiple datasets and model architectures demonstrate the effectiveness of DSM, showing that it broadens the difficulty range of distilled samples and further improves the final performance. Code will be released.
\end{abstract}

\begin{figure}[!t]
    \centering

    \begin{subfigure}[t]{\linewidth}
        \centering
        \includegraphics[width=\linewidth]
        {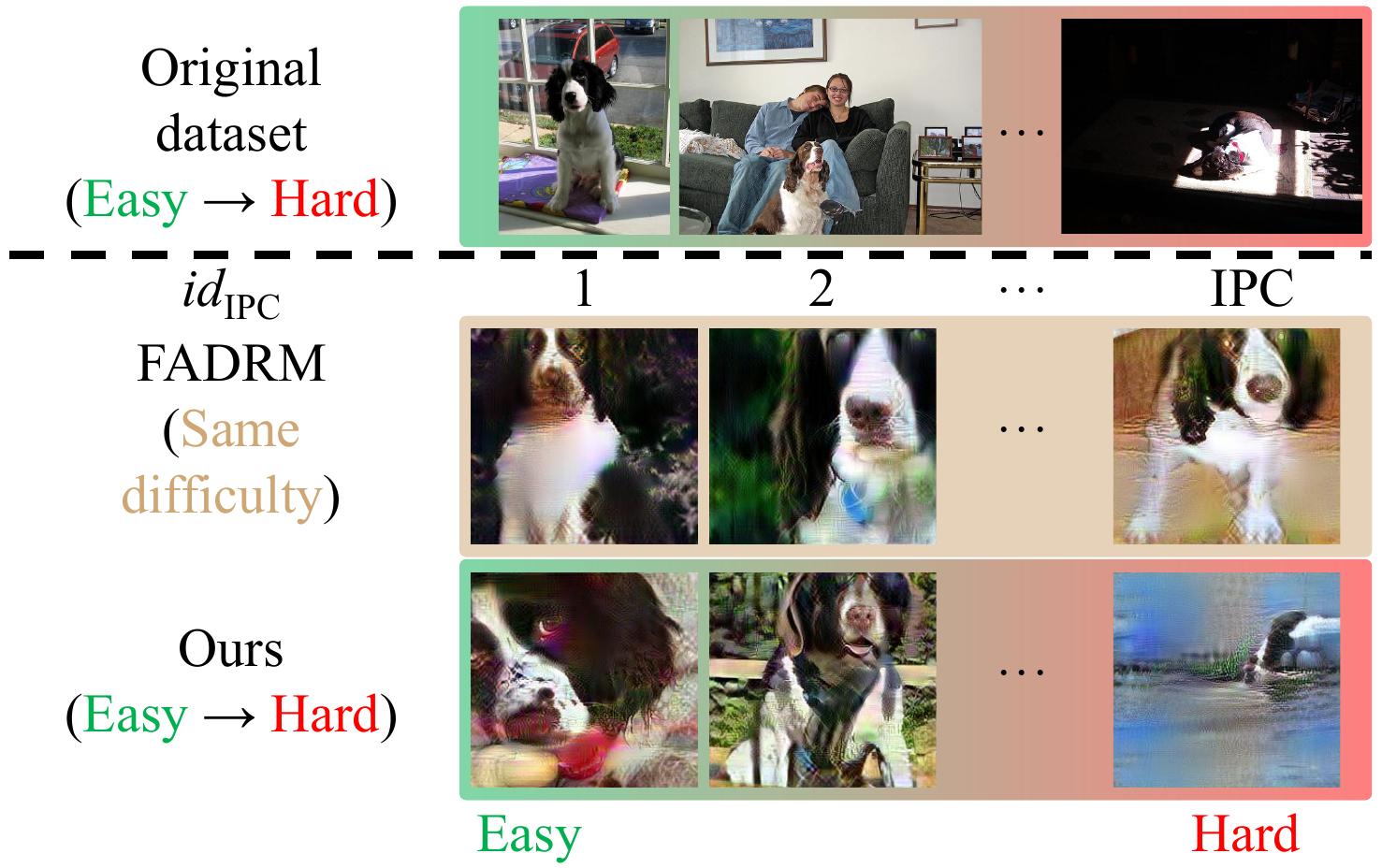}
        \caption{}
        \label{fig:direct_demo}
    \end{subfigure}

    \begin{subfigure}[t]{0.49\linewidth}
        \centering
        \includegraphics[width=\linewidth]
        {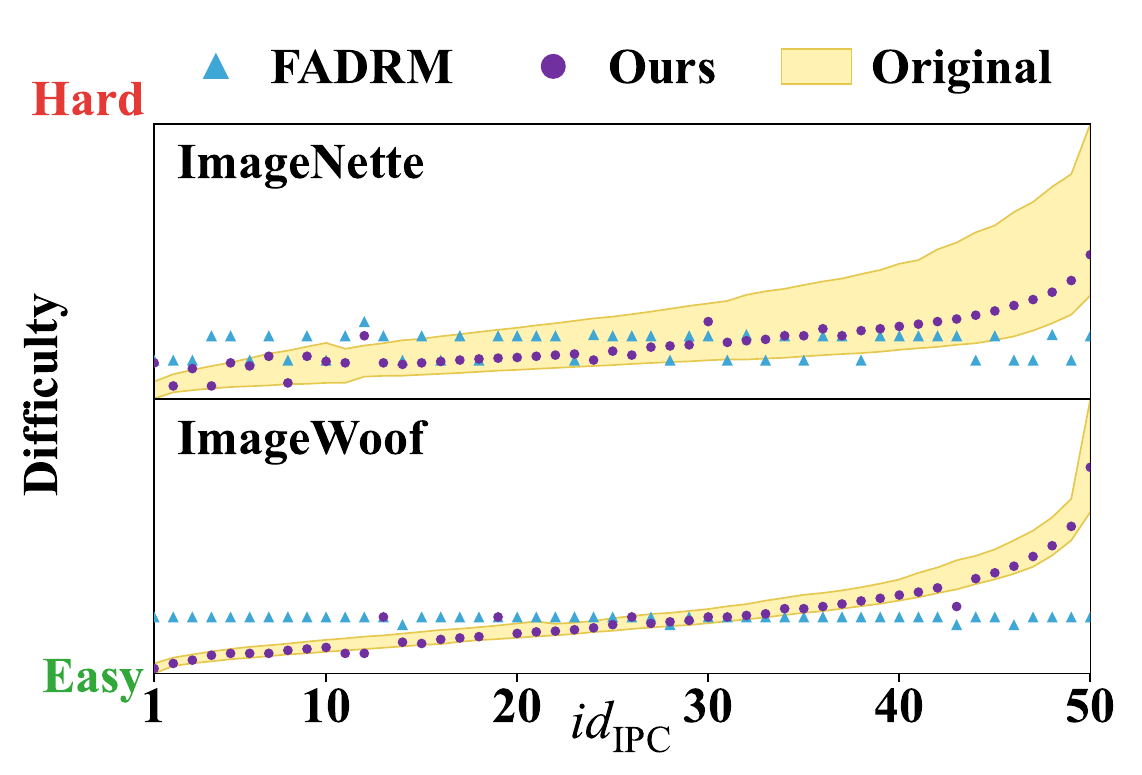}
        \caption{}
        \label{fig:difficulty_ImageNette_ImageWoof}
    \end{subfigure}%
    \hfill
    \begin{subfigure}[t]{0.49\linewidth}
        \centering
        \includegraphics[width=\linewidth]
        {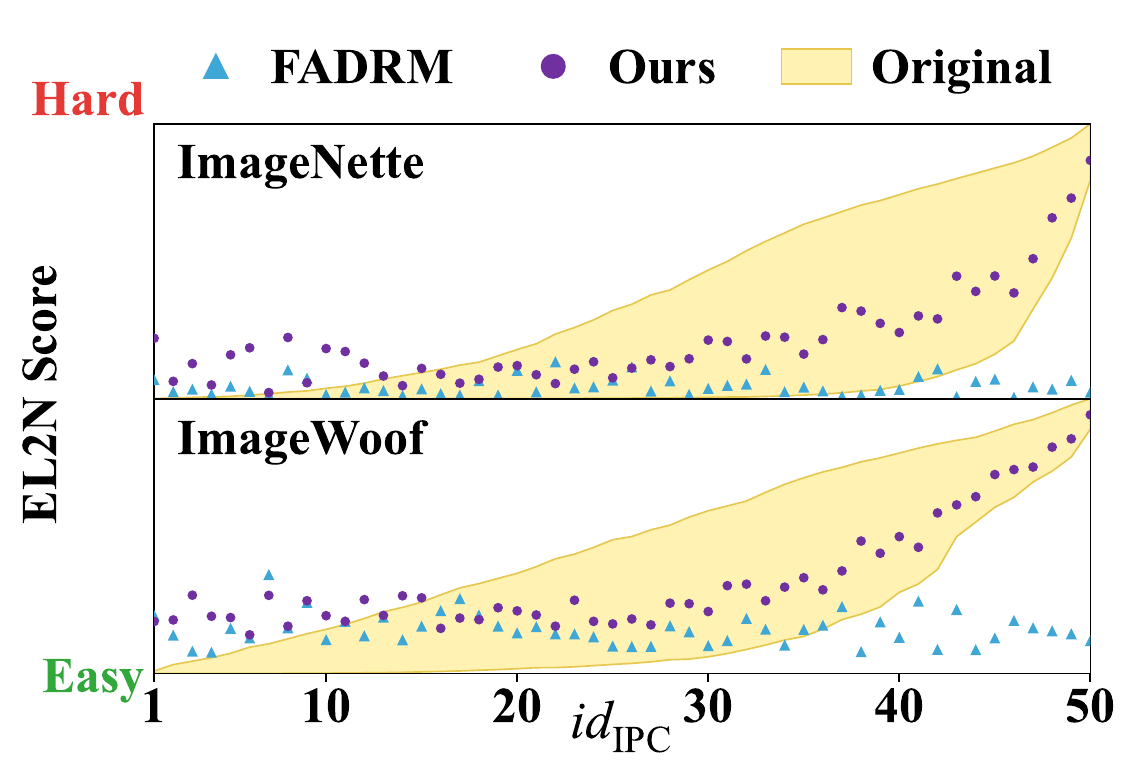}
        \caption{}
        \label{fig:EL2N_difficulty_ImageNette_ImageWoof}
    \end{subfigure}

    \caption{(a) Original English Springer samples in ImageNette range from easy to hard. Ours preserves this structure, while FADRM shows limited difficulty variation. Difficulty distributions on ImageNette/Woof measured by (b) GDS and (c) EL2N~\citep{EL2N}. FADRM's difficulty remains nearly constant across $id_{\mathrm{IPC}}$, while ours gradually increases, following the trend of the original data and covering a broader difficulty range.}
    \label{fig:motivation}
\end{figure}

\section{Introduction}
\label{sec:introduction}
As datasets grow and neural networks advance, the computational resources have increased substantially~\citep{hoffmann2022training,sangermano2022sample}, becoming a major limitation on model development. Dataset distillation (DD) aims to distill a large original set into a small set with high training utility~\citep{lei2024comprehensive,liu2025evolution}, such that models trained on the distilled set can achieve performance comparable to those trained on the original set.

Representative DD methods include gradient matching~\citep{DSA}, distribution matching~\citep{GeoDM}, trajectory matching~\citep{MCT}, generative methods~\citep{SCG}, and decoupled methods~\citep{MIL}. Matching methods often rely on bi-level optimization, while generative methods require large generative models, incurring substantial computational and time costs. Recently, decoupled methods~\citep{SRe2L} based on BN (Batch Normalization) statistics matching~\citep{BN} have received increasing attention. By separating teacher pretraining from data distillation, and using the teacher's BN running statistics to supervise distilled data, they markedly reduce these costs while retaining strong performance.

However, existing decoupled methods typically apply BN running statistics estimated from the entire original set to all distilled data. These global statistics mainly characterize the average feature distribution; uniformly matching them therefore smooths out the representation differences between easy and hard data. As shown in Figure~\ref{fig:direct_demo}, the data distilled by FADRM exhibit visually similar difficulty. Figure~\ref{fig:difficulty_ImageNette_ImageWoof} and~\ref{fig:EL2N_difficulty_ImageNette_ImageWoof} further reveal that their difficulty remains nearly constant. Consequently, the distilled data cannot adequately cover the range of feature distributions spanned by original data of varying difficulty, limiting the student's ability to learn the difficulty structure of the original data.

To address this limitation, we propose \textbf{D}ifficulty-aware \textbf{S}tatistical \textbf{M}atching (\textbf{DSM}). To align difficulty estimation with statistical matching in the same feature space, we carefully design the Global Difficulty Score (GDS): after pretraining, for each original image, GDS computes statistics at the inputs to the teacher's BN layers, and then measures image difficulty as their distance from the corresponding BN running statistics (under the ensemble-based paradigm, DSM averages the GDS rankings of all teachers to integrate their judgments); a higher GDS indicates greater deviation from the average feature distribution, and thus greater difficulty; based on the GDS ranking, DSM then partitions the original data within each class into IPC (Images per Class) difficulty groups.

During distillation, we propose Statistics Updated Again (SUA) to provide precise statistical supervision for each difficulty group: when optimizing the $g$-th batch of distilled data, SUA feeds original data from the $g$-th group into the teacher, and updates its BN running statistics solely through forward passes, ensuring that the supervision accurately reflects the group's difficulty. SUA enables the distilled data to cover a broader difficulty range, helping the student learn a more complete difficulty structure. Additionally, we design Initial Sample Screening (ISS) to select real images from the same group to initialize the distilled data, aligning their initial difficulty with that of the target group.

In most experiments across multiple datasets, DSM outperforms existing state-of-the-art (SOTA) methods, demonstrating its effectiveness. Our contributions are summarized as follows:
\begin{itemize}
    \item We revisit the decoupled DD paradigm, in which BN running statistics estimated from the original set provide supervision, and show that it limits variation in difficulty within the distilled data. To address this issue, we propose GDS, a difficulty metric closely aligned with statistical supervision, to rank the original data by difficulty.
    \item We propose SUA, which uses each group's original data to update the teacher's statistics solely through forward passes, providing each distilled batch with statistical supervision that matches its target group. This enables DSM to adjust supervision signals at low computational cost, broadening the difficulty range of the distilled data.
    \item We conducted experiments on multiple datasets and models, and the results show that DSM broadens the difficulty range of distilled samples, further improving the final performance and demonstrating its effectiveness.
\end{itemize}

\section{Related Work}
Dataset distillation (DD) compresses a large dataset into a small set of samples with high training utility, reducing storage and training costs while retaining performance comparable to training on the full dataset. Existing methods mainly include gradient matching~\citep{DCC,DREAM}, distribution matching~\citep{GUARD,li2025davdd}, trajectory matching~\citep{APM,DATM}, decoupled distillation~\citep{CDA,LPLD}, and generative distillation~\citep{HPD,cai2026evlf}. SRe$^2$L~\citep{SRe2L} pioneered decoupled distillation by separating teacher pretraining from distillation. FADRM~\citep{FADRM} employs multiscale data residual connections to preserve original image information and enrich distilled samples, while FADRM+ extends it to multiple teachers. Further related work is discussed in the supplementary material.

\section{Preliminaries}
\subsection{Dataset Distillation}
Let $\mathcal{T}=\{(x_i,y_i)\}_{i=1}^{N}$ denote the original set, and let $\mathcal{D}=\{(\tilde{x}_j,\tilde{y}_j)\}_{j=1}^{M}$ denote the distilled set, where $M\ll N$. Let $\theta_{\mathcal{T}}$ and $\theta_{\mathcal{D}}$ represent the parameters obtained by training on $\mathcal{T}$ and $\mathcal{D}$, respectively. DD aims to construct $\mathcal{D}$ such that models trained on $\mathcal{D}$ achieve performance comparable to those trained on $\mathcal{T}$. Given a distillation budget $M$, this objective can be formulated as
\begin{equation}
\label{eq:dd}
\mathcal{D}^{*}
=
\operatorname*{arg\,min}_{\mathcal{D}:\,|\mathcal{D}|=M}
\sup_{(x,y)\in\mathcal{T}}
\left|
\mathcal{L}\!\left(f_{\theta_{\mathcal{T}}}(x),y\right)
-
\mathcal{L}\!\left(f_{\theta_{\mathcal{D}}}(x),y\right)
\right|,
\end{equation}
where $\mathcal{L}(\cdot,\cdot)$ denotes the task loss. Under a given distillation budget, the value of IPC is the same across all classes. Typically, a student trained on the distilled set is evaluated on the original test set, and its test performance is used to measure the quality of the distilled set.

\subsection{Decoupled Dataset Distillation}
SRe$^2$L~\citep{SRe2L} introduced a three-stage framework that separates model training from data distillation. During distillation, the pretrained teacher is fixed, and only the distilled data are optimized, avoiding the inner optimization of model parameters and reducing computational overhead.

\noindent\textbf{I. Teacher model pretraining.}
The teacher is trained on the original set $\mathcal{T}$ to learn feature representations, while its BN layers accumulate running means and variances from the training data. The resulting parameters $\theta_{\mathcal{T}}$ and BN statistics are then fixed to guide distillation.

\noindent\textbf{II. Distillation.}
A distilled batch $\widetilde{\mathcal{B}}={(\tilde{x}_j,\tilde{y}_j)}_{j=1}^{B}$ is fed into the frozen teacher, where $B$ is typically equal to the number of classes $C$. The classification loss $\mathcal{L}_{\mathrm{cls}}$ enforces semantic consistency, while the statistics loss $\mathcal{L}_{\mathrm{BN}}$ aligns the batch statistics with the teacher's running statistics:
\begin{equation}
\label{eq:decoupled_distillation}
\begin{aligned}
\mathcal{L}_{\mathrm{dist}}(\widetilde{\mathcal{B}})
&=
\mathcal{L}_{\mathrm{cls}}
\left(\tilde{x},\tilde{y}\right)
+
\lambda_{\mathrm{BN}}
\mathcal{L}_{\mathrm{BN}}
\left(\{\tilde{x}_j\}_{j=1}^{B}\right), \\
\text{where}\quad
\mathcal{L}_{\mathrm{cls}}
\left(\tilde{x},\tilde{y}\right)
&=
\mathbb{E}_{(\tilde{x},\tilde{y})\sim\widetilde{\mathcal{B}}}
\left[
\mathcal{L}_{\mathrm{CE}}
\left(f_{\theta_{\mathcal{T}}}(\tilde{x}),\tilde{y}\right)
\right], \\
\mathcal{L}_{\mathrm{BN}}
\left(\{\tilde{x}_j\}_{j=1}^{B}\right)
&=
{\textstyle\sum_{\ell=1}^{L}}
\left(
\left\|{\mu}_{\ell}-\bar{\mu}_{\ell}\right\|_{2}
+
\left\|{\sigma}_{\ell}^{2}-\bar{\sigma}_{\ell}^{2}\right\|_{2}
\right),
\end{aligned}
\end{equation}
where $\mathcal{L}_{\mathrm{CE}}$ is the cross entropy loss; $\lambda_{\mathrm{BN}}$ weights statistics matching; and $L$ is the number of BN layers, indexed by $\ell$. For a BN layer with $C_{\ell}$ channels, $\mu_{\ell},\sigma_{\ell}^{2}\in\mathbb{R}^{C_{\ell}}$ are the batch mean and variance at its input, while $\bar{\mu}_{\ell},\bar{\sigma}_{\ell}^{2}\in\mathbb{R}^{C_{\ell}}$ are the corresponding teacher running statistics.

\noindent\textbf{III. Soft label generation.}
The distilled data are fed into the pretrained teacher, whose class probability distributions are used as soft labels. These labels also preserve semantic relationships for all classes. The student is trained on the distilled data and soft labels, receiving richer supervision.

\subsection{Batch Normalization}
Batch Normalization (BN)~\citep{BN} was originally introduced to mitigate changes in intermediate feature distributions. For features entering a BN layer, BN first computes the mean and variance of the current batch for each channel, and uses them to normalize the features. This process stabilizes feature scales across layers and improves stability.

During training, each BN layer uses the current batch statistics for feature normalization, and its running mean and variance are continuously updated through an exponential moving average. Given the training batch $\mathcal{B}_{e}$ at step $e$, the running statistics of BN layer $\ell$ are updated as follows:
\begin{equation}
\label{eq:bn_running_statistics}
\begin{aligned}
\bar{\mu}_{\ell,e}
&=
(1-\rho)\bar{\mu}_{\ell,e-1}
+
\rho\mu_{\ell}(\mathcal{B}_{e}),\\
\bar{\sigma}_{\ell,e}^{2}
&=
(1-\rho)\bar{\sigma}_{\ell,e-1}^{2}
+
\rho\sigma_{\ell}^{2}(\mathcal{B}_{e}).
\end{aligned}
\end{equation}
where $\rho\in(0,1]$ controls the statistics update, a larger $\rho$ assigns more weight to the current batch. A commonly used value is $\rho=0.1$.

\section{Method}

\subsection{Limitations of Decoupled Dataset Distillation}

SRe$^2$L~\citep{SRe2L} pioneered the decoupled DD paradigm, which uses the BN running statistics stored in a pretrained teacher as supervisory signals to guide the optimization, as shown in ~\eqref{eq:decoupled_distillation}. Because of its efficiency, this paradigm has attracted considerable attention. Recent studies have improved this paradigm by reducing storage overhead~\citep{HLI}, lowering computational complexity~\citep{MIL}, and enhancing performance~\citep{BPS}. However, most of these studies rely on additional components or optimization mechanisms, with limited attention paid to the fundamental supervisory signals and their influence.

\noindent\textbf{Homogeneous Supervision.} Existing decoupled methods use the global BN running statistics stored in the pretrained teacher as the target for every distilled batch. Let \(\boldsymbol{\tau}_g\) denote the statistical supervision assigned to batch \(g\), and let \(\boldsymbol{\tau}^{\mathrm G}\) denote the global BN statistics. Then
\[
\tau_g=\tau^{\mathrm{G}},\ \forall g \quad\Longrightarrow\quad \mathcal{D}_{\mathrm{sup}}(\tau_g,\tau_h)=0,\ \forall g,h.
\]
Consequently, the statistical target is independent of the batch index, and cannot explicitly provide batch-specific difficulty supervision. Although initialization and optimization may introduce difficulty variation, this variation is not controlled by the supervision signal, making it difficult to preserve the difficulty structure of the original data, as illustrated in Figure~\ref{fig:difficulty_ImageNette_ImageWoof} and~\ref{fig:EL2N_difficulty_ImageNette_ImageWoof}.

\subsection{Difficulty-aware Statistical Matching}
Figure~\ref{fig:pipeline} illustrates the overall pipeline of the proposed method. Given the original set and a pretrained teacher, we first employ GDS to evaluate the difficulty of each original data point and rank the data within each class from easy to hard. We then divide the ordered data of each class into \(\mathrm{IPC}\) difficulty groups, with different groups representing data of varying difficulty in the original set. When distilling the \(g\)-th batch, SUA feeds the original data in difficulty group \(g\) into the teacher to further adapt its BN running statistics to the characteristics of the current group, thereby constructing distinct statistical supervision signals for the corresponding distilled batch. By matching the respective statistical supervision signals, the distilled data in different batches progressively learn original data patterns from easy to hard. Meanwhile, ISS selects original data from the corresponding difficulty group to initialize the distilled data, and provides a favorable starting point for subsequent optimization. GDS, SUA, and ISS are introduced in the following sections. See the supplementary material for pseudocode.

\begin{figure*}[t]
    \centering
    \includegraphics[width=\textwidth]{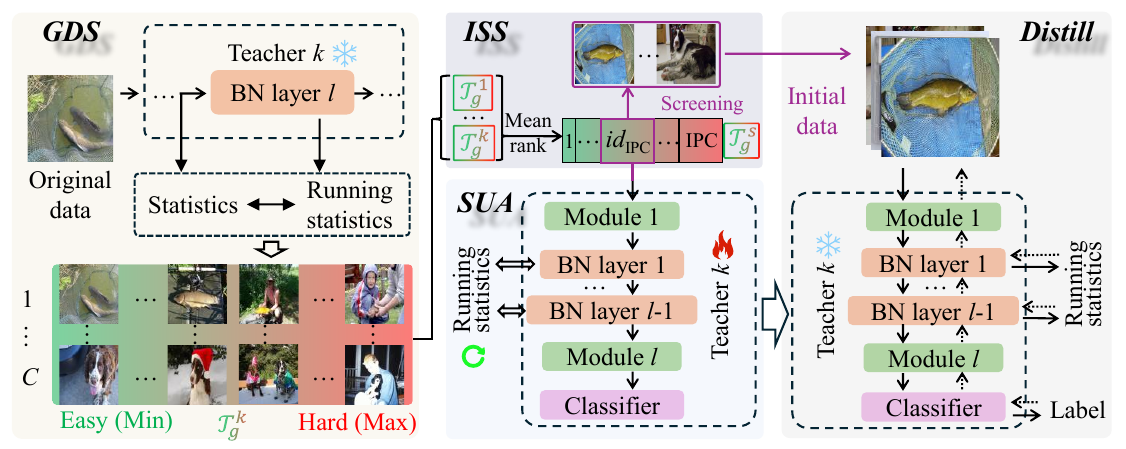}
    \caption{Overview of DSM. First, GDS evaluates the difficulty of the original data and ranks them from easy to hard within each class. The ordered data are then divided into \(\mathrm{IPC}\) difficulty groups. When distilling the \(g\)-th batch, SUA feeds the original data in difficulty group \(g\) into the teacher for forward propagation, allowing the statistics of each BN layer to characterize the data features. Meanwhile, ISS selects original data from the same difficulty group to initialize the corresponding distilled data, thereby providing a favorable starting point.}
    \label{fig:pipeline}
\end{figure*}

\subsubsection{GDS: Global Difficulty Score}
To accurately partition the original set by difficulty and ensure that difficulty estimation and BN statistics operate in the same feature space, we propose a new difficulty partitioning method, GDS.

\noindent\textbf{Operation and Purpose.}
We compute the per-channel mean and variance of $x_i$ at the input to each BN layer, and compare them with the corresponding running mean and variance. GDS is defined as
\begin{equation}
\label{eq:gps}
\begin{aligned}
\operatorname{GDS}(x_i)
&=\mathbb{E}_{k\in\{1,\ldots,K\}}
\left[
\operatorname{rank}_{c}
\left(
\mathbb{E}_{\ell\in\{1,\ldots,L_k\}}
[d_{\ell,i,k}]
\right)
\right],\\
d_{\ell,i,k}
&=\frac{1}{\sqrt{C_{\ell,k}}}
\Bigg(
\left\|\bm{\mu}_{\ell,i,k}-\bar{\bm{\mu}}_{\ell,k}\right\|_2\\
&\qquad+\beta
\left\|
\log\frac{\bm{\sigma}_{\ell,i,k}^{2}+\epsilon}
{\bar{\bm{\sigma}}_{\ell,k}^{2}+\epsilon}
\right\|_2
\Bigg).
\end{aligned}
\end{equation}
Here $d_{\ell,i,k}$ denotes the normalized statistical deviation at BN layer $\ell$, $K$ is the number of teachers, $\operatorname{rank}_{c}(\cdot)$ is the rank within class $c$, $\beta$ balances the mean and variance distances, and $\epsilon$ ensures numerical stability. For the variance, we first compute the ratio between the sample variance and the BN running variance to reduce the effect of variance scale differences across channels. We then take the logarithm of this ratio to stabilize its numerical range. Finally, we normalize the statistical distance at each layer by $\sqrt{C_{\ell,k}}$ to reduce distance scale discrepancies caused by the number of channels across BN layers.

A larger GDS indicates that the data exhibits more pronounced statistical deviations across multiple feature layers, and thus corresponds to greater difficulty. Based on the GDS ranking, DSM divides the ordered data of each class into IPC difficulty groups of approximately equal size, which sequentially represent the original data from easy to hard. The resulting difficulty groups provide the basis for SUA to construct distinct supervision, and for ISS to select the corresponding initialization data.

\noindent\textbf{Observation.}
A larger GDS indicates more pronounced deviations between the sample statistics and the teacher's BN running statistics. To verify the correlation between statistics deviations and classification difficulty, we sort the original samples in ascending order of GDS, and compute the teacher's cross-entropy loss on each sample's ground truth label. The observed relationship between GDS and classification difficulty is expressed as
\begin{equation}
\operatorname{GDS}(x_i)\uparrow
\quad\Longrightarrow\quad
\mathcal{L}_{\mathrm{CE}}\!\left(f_{\theta}(x_i),y_i\right)
=
-\log p_{\theta}(y_i\mid x_i)\uparrow.
\end{equation}
We observe that the cross-entropy loss generally increases along the GDS ranking. Samples with larger GDS values tend to have higher classification losses. These results indicate that GDS effectively reflects the sample difficulty perceived by the teacher, providing a basis for grouping samples from easy to hard. A detailed analysis is provided in the supplementary material.

\subsubsection{SUA: Statistics Updated Again}

To convert the difficulty groups obtained by GDS into statistical supervision corresponding to individual distillation batches, we propose SUA.

\noindent\textbf{Operation and Purpose.}
Let $\mathcal{T}_g$ denote the $g$-th difficulty group. We set $\mathrm{IPC}$ difficulty groups, with the $g$-th group corresponding to the $id_{\mathrm{IPC}}$-th distillation batch. For teacher $k$, SUA starts from the global BN running statistics $\boldsymbol{\tau}_k^{\mathrm G}$ stored in the pretrained teacher, freezes the teacher parameters $\theta_k$, and performs $\left\lfloor r_{\mathrm{fwd}}E_k^{\mathrm{pre}}\right\rfloor$ epochs of forward propagation over $\mathcal{T}_g$. During this process, only the BN running statistics are updated according to ~\eqref{eq:bn_running_statistics}. This process is expressed as
\begin{equation}
\boldsymbol{\tau}_{k}^{\mathrm{G}}
\mathrel{\raisebox{-0.7ex}{$
\xrightarrow{
f_{\theta_k=\mathrm{constant}}^{
\left\lfloor r_{\mathrm{fwd}}E_k^{\mathrm{pre}}\right\rfloor
}(\mathcal{T}_g)
}
$}}
\boldsymbol{\tau}_{k,g}^{\mathrm{SUA}}.
\end{equation}
where $r_{\mathrm{fwd}}$ denotes the ratio of forward propagation epochs, $E_k^{\mathrm{pre}}$ denotes the number of pretraining epochs for teacher $k$, $\left\lfloor r_{\mathrm{fwd}}E_k^{\mathrm{pre}}\right\rfloor$ denotes the actual number of forward propagation epochs, and $\boldsymbol{\tau}_{k,g}^{\mathrm{SUA}}$ denotes the BN running statistics obtained by SUA.

Through this update, the BN running statistics of the teacher gradually approach those of the current difficulty group, and serve as the statistical matching signal for distillation batch $g$, thereby providing each distillation batch with supervision signal corresponding to its difficulty.

\begin{proposition}[Expected Convergence; Proof in Supplementary Material]
\label{prop:expected_convergence}
For any fixed teacher $k$, let $\boldsymbol{\tau}_{k,g}^{\star}$ and $\boldsymbol{\tau}_{k,g}^{(t)}$ denote the expected BN batch statistics of $\mathcal{T}_g$ and the statistics after $t$ forward batch updates, respectively. If the forward batches are independently drawn from $\mathcal{T}_g$, then
\begin{equation}
\mathbb{E}\left[\boldsymbol{\tau}_{k,g}^{(t)}\right]
-\boldsymbol{\tau}_{k,g}^{\star}
=
(1-\rho)^t
\left(
\boldsymbol{\tau}_k^{\mathrm G}
-\boldsymbol{\tau}_{k,g}^{\star}
\right)
\xrightarrow{t\rightarrow\infty}
\boldsymbol{0}.
\end{equation}
Thus, SUA statistics converge in expectation to those of the current difficulty group. The updated statistics then serve as the matching signal for batch $g$, guiding its distilled data to match the corresponding feature statistics.
\end{proposition}

\subsubsection{ISS: Initial Sample Screening}

The initialization of distilled data affects final performance. EDC~\citep{EDC} shows that initialization with real images can reduce the distribution gap between the distilled data and the original data. SelMatch~\citep{SelMatch} shows that selecting real samples of appropriate difficulty helps incorporate corresponding features. Based on these observations, we adopt ISS as an auxiliary initialization strategy to align the initial samples of each distilled batch with its target difficulty.

Specifically, let $\mathcal{T}_{g,c}$ denote the original data of class $c$ in difficulty group $g$. ISS further selects real images from $\mathcal{T}_{g,c}$ based on their difficulty, and uses them to initialize the distilled data $\widetilde{x}_{g,c}^{(0)}$. The initial images and the statistical supervision come from the same difficulty group, helping reduce statistical discrepancies early in optimization. Experimental results show that ISS further improves performance.

\section{Experiments}
\subsection{Experimental Setup}
\noindent\textbf{Baselines.}
FADRM~\citep{FADRM} is a recent SOTA method for decoupled DD, achieving improved performance while substantially reducing memory and time costs, we therefore adopt it as our primary baseline and further compare with FADRM+, its extension using multiple models. SRe$^2$L~\citep{SRe2L}, a representative foundational method for decoupled DD, is also included for comparison. In addition, we include RDED~\citep{RDED} as a representative method from other DD paradigms, it constructs distilled data by assembling real images, which has been adopted by other DD methods for distilled data initialization.

\noindent\textbf{Datasets.}
We evaluate DSM on benchmark datasets widely used in dataset distillation. For low resolution datasets, we use CIFAR-10/100~\citep{CIFAR-10/100} at $32\times32$, and Tiny-ImageNet~\citep{Tiny-ImageNet} at $64\times64$. For high-resolution datasets, we mainly use ImageNet-1k~\citep{ImageNet} at $224\times224$ and its subsets, ImageWoof (fine-grained dataset) and ImageNette.

\noindent\textbf{Implementation Details.}
We reproduce FADRM using its provided settings. For DSM, all settings remain the same except for method-specific hyperparameters. For a fair comparison with FADRM+, we further adopt the multi-teacher variant DSM+, which consistently employs ShuffleNetV2~\citep{ShuffleNet}, ResNet18~\citep{ResNet}, MobileNetV2~\citep{MobileNetV2}, and DenseNet121~\citep{DenseNet} as teachers for distillation across all datasets. The remaining parameters are kept consistent with FADRM, while other baselines follow their original settings. See the supplementary material for more details.

\subsection{Comparison With State-of-the-Art Methods}
\subsubsection{Main Results}
In Table~\ref{tab:main_results}, methods marked with ``+'' use multiple teachers for distillation and a single student for evaluation, while the remaining methods follow the standard same-architecture evaluation protocol.

\noindent\textbf{Low-resolution datasets.} As shown in Table~\ref{tab:main_results}, we evaluate DSM on CIFAR-10/100 and Tiny-ImageNet. With ResNet18/50 as students, DSM(+) outperforms the corresponding baselines in most settings. On CIFAR-10 at $\mathrm{IPC}=10$, DSM achieves $52.7\%$ accuracy with ResNet50, an improvement of $8.9\%$. At $\mathrm{IPC}=1$, all original data form a single difficulty group, so SUA cannot provide differentiated supervision across groups, leaving ISS as the main source of performance differences. For low resolution data, limited useful differences among initialization samples may restrict the benefits of ISS, potentially explaining the similar performance of DSM and the baselines at $\mathrm{IPC}=1$, especially on CIFAR-10.

\noindent\textbf{High-resolution datasets.}
We evaluate DSM on ImageNette, ImageWoof, and ImageNet-1k. As shown in Table~\ref{tab:main_results}, DSM achieves consistent performance gains in most settings, demonstrating its strong robustness. Performance degradation occurs mainly on ImageWoof with ResNet50 as the student under the $\mathrm{IPC}=1/10$. One possible explanation is the limited compatibility of ResNet50 with ImageWoof: when trained on the original set, ResNet18 achieves an accuracy of 85.0\%, whereas ResNet50 achieves only 80.3\%. Under low IPC, the limited distilled data may further amplify this mismatch, making it difficult for ResNet50 to learn the difficulty information preserved by DSM.

\begin{table*}[t]
\centering
\caption{Top-1 accuracy comparison of DSM, extended variant DSM+, and other SOTA methods. For each dataset, IPC, and student setting, the mean values of the \textbf{best} and \underline{second best} results are marked in \textbf{bold} and \underline{underlined}, respectively. Values in parentheses after the DSM and DSM+ results indicate their performance changes relative to FADRM and FADRM+, respectively.}
\label{tab:main_results}

\small
\setlength{\tabcolsep}{2.2pt}
\renewcommand{\arraystretch}{1.15}

\resizebox{\textwidth}{!}{%
\begin{tabular}{
@{}cc
*{2}{
    cc
    >{\columncolor{FADRMBlue}}c
    >{\columncolor{PSMPurple}}c
    >{\columncolor{FADRMBlue}}c
    >{\columncolor{PSMPurple}}c
}
@{}
}
\toprule
\multicolumn{2}{c}{Student} & \multicolumn{6}{c}{ResNet18} & \multicolumn{6}{c}{ResNet50~\citep{ResNet}} \\
\cmidrule(lr){1-2}
\cmidrule(lr){3-8}
\cmidrule(lr){9-14}

Dataset & IPC
& SRe$^2$L & RDED & FADRM & DSM & FADRM+ & DSM+ & SRe$^2$L & RDED & FADRM & DSM & FADRM+ & DSM+ \\
\midrule

\multirow{4}{*}{CIFAR-10}
& 1 & \tbscore{16.6}{0.9} & \tbscore{\underline{22.9}}{0.4} & \tbscore{19.3}{0.6} & \tbscore{20.9}{0.1}\psmup{1.6} & \tbscore{\textbf{23.7}}{0.8} & \tbscore{22.3}{0.6}\psmdown{1.4} & \tbscore{15.2}{1.3} & \tbscore{19.7}{1.7} & \tbscore{23.2}{0.7} & \tbscore{23.2}{0.2}\psmeq{0.0} & \tbscore{\textbf{23.5}}{1.3} & \tbscore{\underline{23.3}}{0.4}\psmdown{0.2} \\

& 10 & \tbscore{29.3}{0.5} & \tbscore{37.1}{0.3} & \tbscore{48.2}{0.4} & \tbscore{54.5}{1.3}\psmup{6.3} & \tbscore{\underline{55.9}}{1.0} & \tbscore{\textbf{61.1}}{0.8}\psmup{5.2} & \tbscore{30.3}{1.7} & \tbscore{32.5}{0.9} & \tbscore{43.8}{1.1} & \tbscore{52.7}{0.1}\psmup{8.9} & \tbscore{\underline{55.1}}{0.8} & \tbscore{\textbf{56.8}}{1.1}\psmup{1.7} \\

& 50 & \tbscore{45.0}{0.7} & \tbscore{62.1}{0.1} & \tbscore{80.6}{0.7} & \tbscore{81.1}{0.1}\psmup{0.5} & \tbscore{\underline{85.8}}{0.1} & \tbscore{\textbf{86.7}}{0.2}\psmup{0.9} & \tbscore{52.9}{1.3} & \tbscore{52.5}{2.0} & \tbscore{79.7}{1.2} & \tbscore{82.7}{0.7}\psmup{3.0} & \tbscore{\underline{83.7}}{0.6} & \tbscore{\textbf{86.0}}{0.9}\psmup{2.3} \\

\cmidrule(lr){2-14} & Whole dataset & \multicolumn{6}{c}{92.9} & \multicolumn{6}{c}{93.1} \\
\midrule

\multirow{4}{*}{CIFAR-100}
& 1 & \tbscore{6.6}{0.2} & \tbscore{11.1}{0.3} & \tbscore{31.3}{0.2} & \tbscore{32.9}{0.9}\psmup{1.6} & \tbscore{\underline{37.9}}{0.8} & \tbscore{\textbf{39.1}}{0.1}\psmup{1.2} & \tbscore{6.0}{0.1} & \tbscore{11.6}{0.4} & \tbscore{24.4}{1.6} & \tbscore{26.2}{1.4}\psmup{1.8} & \tbscore{\underline{33.6}}{0.8} & \tbscore{\textbf{34.9}}{1.3}\psmup{1.3} \\

& 10 & \tbscore{27.0}{0.4} & \tbscore{42.6}{0.2} & \tbscore{64.7}{0.4} & \tbscore{\textbf{66.7}}{0.4}\psmup{2.0} & \tbscore{62.5}{0.2} & \tbscore{\underline{66.1}}{0.2}\psmup{3.6} & \tbscore{35.4}{1.9} & \tbscore{50.3}{0.4} & \tbscore{61.8}{0.5} & \tbscore{\underline{64.6}}{0.1}\psmup{2.8} & \tbscore{62.3}{0.2} & \tbscore{\textbf{65.7}}{0.4}\psmup{3.4} \\

& 50 & \tbscore{50.2}{0.4} & \tbscore{62.6}{0.1} & \tbscore{69.3}{0.3} & \tbscore{\underline{69.7}}{0.3}\psmup{0.4} & \tbscore{67.4}{0.4} & \tbscore{\textbf{70.1}}{0.1}\psmup{2.7} & \tbscore{52.9}{0.1} & \tbscore{66.8}{0.3} & \tbscore{67.6}{0.5} & \tbscore{\underline{69.8}}{0.1}\psmup{2.2} & \tbscore{68.2}{0.1} & \tbscore{\textbf{70.7}}{0.2}\psmup{2.5} \\

\cmidrule(lr){2-14}
& Whole dataset & \multicolumn{6}{c}{73.6} & \multicolumn{6}{c}{74.4} \\
\midrule

\multirow{4}{*}{Tiny-ImageNet}
& 1 & \tbscore{2.6}{0.1} & \tbscore{9.7}{0.4} & \tbscore{28.6}{0.1} & \tbscore{30.6}{1.3}\psmup{2.0} & \tbscore{\underline{32.4}}{0.2} & \tbscore{\textbf{35.4}}{0.7}\psmup{3.0} & \tbscore{5.1}{0.1} & \tbscore{6.5}{0.5} & \tbscore{31.1}{0.4} & \tbscore{\underline{31.3}}{0.7}\psmup{0.2} & \tbscore{31.0}{1.3} & \tbscore{\textbf{32.9}}{0.9}\psmup{1.9} \\

& 10 & \tbscore{16.1}{0.2} & \tbscore{41.9}{0.2} & \tbscore{46.5}{0.4} & \tbscore{\underline{48.1}}{0.3}\psmup{1.6} & \tbscore{47.3}{1.1} & \tbscore{\textbf{48.4}}{0.3}\psmup{1.1} & \tbscore{43.0}{0.5} & \tbscore{36.9}{0.4} & \tbscore{47.5}{0.3} & \tbscore{\underline{48.0}}{0.3}\psmup{0.5} & \tbscore{47.1}{0.4} & \tbscore{\textbf{48.2}}{0.9}\psmup{1.1} \\

& 50 & \tbscore{41.1}{0.4} & \tbscore{\textbf{58.2}}{0.1} & \tbscore{51.4}{0.1} & \tbscore{51.7}{0.1}\psmup{0.3} & \tbscore{52.0}{1.0} & \tbscore{\underline{52.1}}{0.1}\psmup{0.1} & \tbscore{55.3}{0.1} & \tbscore{48.0}{0.8} & \tbscore{\textbf{57.8}}{0.1} & \tbscore{\underline{56.6}}{0.7}\psmdown{1.2} & \tbscore{52.2}{1.8} & \tbscore{53.0}{0.2}\psmup{0.8} \\

\cmidrule(lr){2-14} & Whole dataset & \multicolumn{6}{c}{60.3} & \multicolumn{6}{c}{63.4} \\
\midrule

\multirow{4}{*}{ImageNette}
& 1 & \tbscore{19.1}{1.1} & \tbscore{\textbf{35.8}}{1.0} & \tbscore{28.2}{0.6} & \tbscore{25.0}{0.5}\psmdown{3.2} & \tbscore{30.5}{1.3} & \tbscore{\underline{32.5}}{0.5}\psmup{2.0} & \tbscore{13.6}{0.5} & \tbscore{23.6}{1.4} & \tbscore{23.7}{0.4} & \tbscore{22.9}{0.8}\psmdown{0.8} & \tbscore{\underline{25.5}}{0.7} & \tbscore{\textbf{29.2}}{0.6}\psmup{3.7} \\

& 10 & \tbscore{29.4}{3.0} & \tbscore{61.4}{0.4} & \tbscore{64.0}{0.1} & \tbscore{67.1}{0.2}\psmup{3.1} & \tbscore{\underline{68.5}}{0.4} & \tbscore{\textbf{69.4}}{0.8}\psmup{0.9} & \tbscore{32.5}{1.1} & \tbscore{52.6}{2.8} & \tbscore{60.9}{0.5} & \tbscore{62.4}{0.4}\psmup{1.5} & \tbscore{\underline{66.3}}{0.8} & \tbscore{\textbf{67.6}}{0.5}\psmup{1.3} \\

& 50 & \tbscore{40.9}{0.3} & \tbscore{80.4}{0.4} & \tbscore{82.8}{0.6} & \tbscore{84.0}{0.3}\psmup{1.2} & \tbscore{\underline{84.6}}{0.7} & \tbscore{\textbf{85.1}}{0.3}\psmup{0.5} & \tbscore{60.8}{1.0} & \tbscore{74.5}{2.5} & \tbscore{81.6}{0.4} & \tbscore{82.3}{0.3}\psmup{0.7} & \tbscore{\underline{85.4}}{0.3} & \tbscore{\textbf{85.9}}{0.2}\psmup{0.5} \\

\cmidrule(lr){2-14} & Whole dataset & \multicolumn{6}{c}{90.2} & \multicolumn{6}{c}{89.1} \\
\midrule

\multirow{4}{*}{ImageWoof}
& 1 & \tbscore{13.3}{0.5} & \tbscore{\textbf{20.8}}{1.2} & \tbscore{19.3}{0.3} & \tbscore{\underline{20.2}}{0.7}\psmup{0.9} & \tbscore{19.4}{0.5} & \tbscore{14.9}{0.9}\psmdown{4.5} & \tbscore{12.2}{2.2} & \tbscore{\underline{21.1}}{1.9} & \tbscore{\textbf{21.6}}{0.5} & \tbscore{16.5}{0.7}\psmdown{5.1} & \tbscore{18.4}{0.2} & \tbscore{15.1}{0.7}\psmdown{3.3} \\

& 10 & \tbscore{20.2}{0.2} & \tbscore{38.5}{2.1} & \tbscore{43.8}{0.2} & \tbscore{44.5}{0.5}\psmup{0.7} & \tbscore{\underline{47.2}}{0.2} & \tbscore{\textbf{48.3}}{0.6}\psmup{1.1} & \tbscore{19.8}{1.4} & \tbscore{40.9}{3.1} & \tbscore{37.8}{0.6} & \tbscore{34.0}{1.0}\psmdown{3.8} & \tbscore{\textbf{44.9}}{0.9} & \tbscore{\underline{43.5}}{1.2}\psmdown{1.4} \\

& 50 & \tbscore{23.3}{0.3} & \tbscore{68.5}{0.7} & \tbscore{67.7}{1.0} & \tbscore{69.6}{0.9}\psmup{1.9} & \tbscore{\underline{73.0}}{0.6} & \tbscore{\textbf{74.4}}{0.5}\psmup{1.4} & \tbscore{35.7}{0.4} & \tbscore{59.8}{3.2} & \tbscore{66.5}{1.1} & \tbscore{68.0}{1.0}\psmup{1.5} & \tbscore{\underline{72.7}}{0.7} & \tbscore{\textbf{73.7}}{0.7}\psmup{1.0} \\

\cmidrule(lr){2-14} & Whole dataset & \multicolumn{6}{c}{85.0} & \multicolumn{6}{c}{80.3} \\
\midrule

\multirow{3}{*}{ImageNet-1k}

& 10 & \tbscore{21.3}{0.6} & \tbscore{42.0}{0.1} & \tbscore{47.8}{0.4} & \tbscore{48.8}{0.2}\psmup{1.0} & \tbscore{\underline{50.8}}{0.1} & \tbscore{\textbf{51.6}}{0.3}\psmup{0.8} & \tbscore{28.4}{0.1} & -- & \tbscore{45.6}{1.6} & \tbscore{46.9}{0.3}\psmup{1.3} & \tbscore{\underline{56.9}}{0.2} & \tbscore{\textbf{58.1}}{0.3}\psmup{1.2} \\

& 50 & \tbscore{46.8}{0.2} & \tbscore{56.5}{0.1} & \tbscore{\underline{60.9}}{0.1} & \tbscore{\textbf{61.3}}{0.1}\psmup{0.4} & \tbscore{59.2}{0.2} & \tbscore{59.8}{0.1}\psmup{0.6} & \tbscore{55.6}{0.3} & -- & \tbscore{65.0}{0.1} & \tbscore{66.3}{0.3}\psmup{1.3} & \tbscore{\underline{66.4}}{0.3} & \tbscore{\textbf{66.6}}{0.1}\psmup{0.2} \\

\cmidrule(lr){2-14} & Whole dataset & \multicolumn{6}{c}{69.8} & \multicolumn{6}{c}{76.1} \\

\bottomrule
\end{tabular}%
}
\end{table*}

\subsubsection{Difficulty and Efficiency Analysis}

\noindent\textbf{Difficulty.} We present the difficulty distributions across datasets in Figures~\ref{fig:difficulty_ImageNette_ImageWoof},~\ref{fig:difficulty_CIFAR-10_CIFAR-100} and~\ref{fig:difficulty_Tiny-ImageNet_ImageNet-1k}. To avoid relying solely on the difficulty metric used by DSM, we further introduce the EL2N score~\citep{EL2N}, based on model prediction probabilities, for additional evaluation. Under both metrics, the data distilled by DSM closely follow the difficulty trends of the original data and cover a broad difficulty range, while the difficulty of the data distilled by FADRM remains nearly constant. These results demonstrate that DSM effectively preserves the difficulty structure of the original data across different types of datasets, validating the robustness and effectiveness of DSM.

\noindent\textbf{Efficiency.}
To evaluate the additional overhead of GDS computation and SUA forward passes, we measure their runtime and peak GPU memory. Since DSM and FADRM use the same distillation process, their efficiency difference mainly arises from these two steps. The costs of SUA and distillation are measured for the distilled batch corresponding to a single $id_{\mathrm{IPC}}$. As shown in Figure~\ref{fig:efficiency}, GDS, SUA, and distillation require similar amounts of memory on small datasets, while distillation requires substantially more memory on Tiny-ImageNet. GDS and SUA also require much less time than distillation because they involve no gradient computation. Therefore, DSM introduces limited additional overhead while maintaining high computational efficiency and improving performance.

\begin{figure}[!t]
    \centering

    \begin{subfigure}[t]{0.49\linewidth}
        \centering
        \includegraphics[width=\linewidth]
        {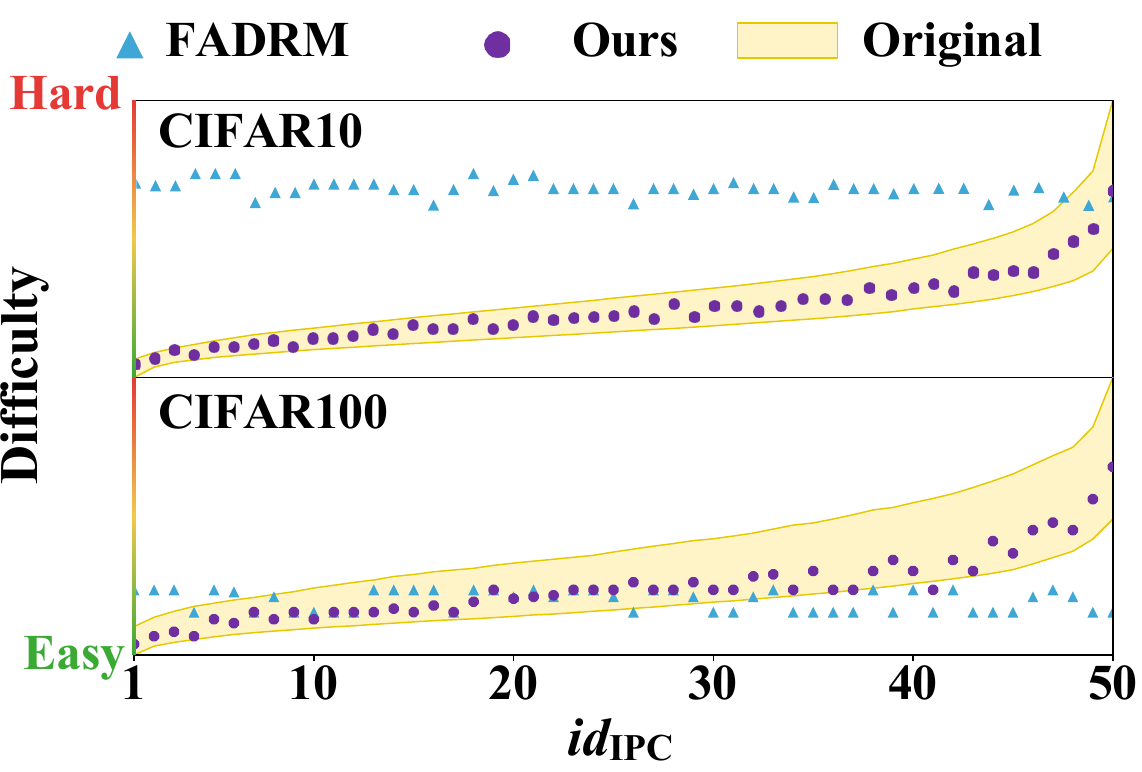}
        \caption{}
        \label{fig:difficulty_CIFAR-10_CIFAR-100}
    \end{subfigure}%
    \hfill
    \begin{subfigure}[t]{0.49\linewidth}
        \centering
        \includegraphics[width=\linewidth]
        {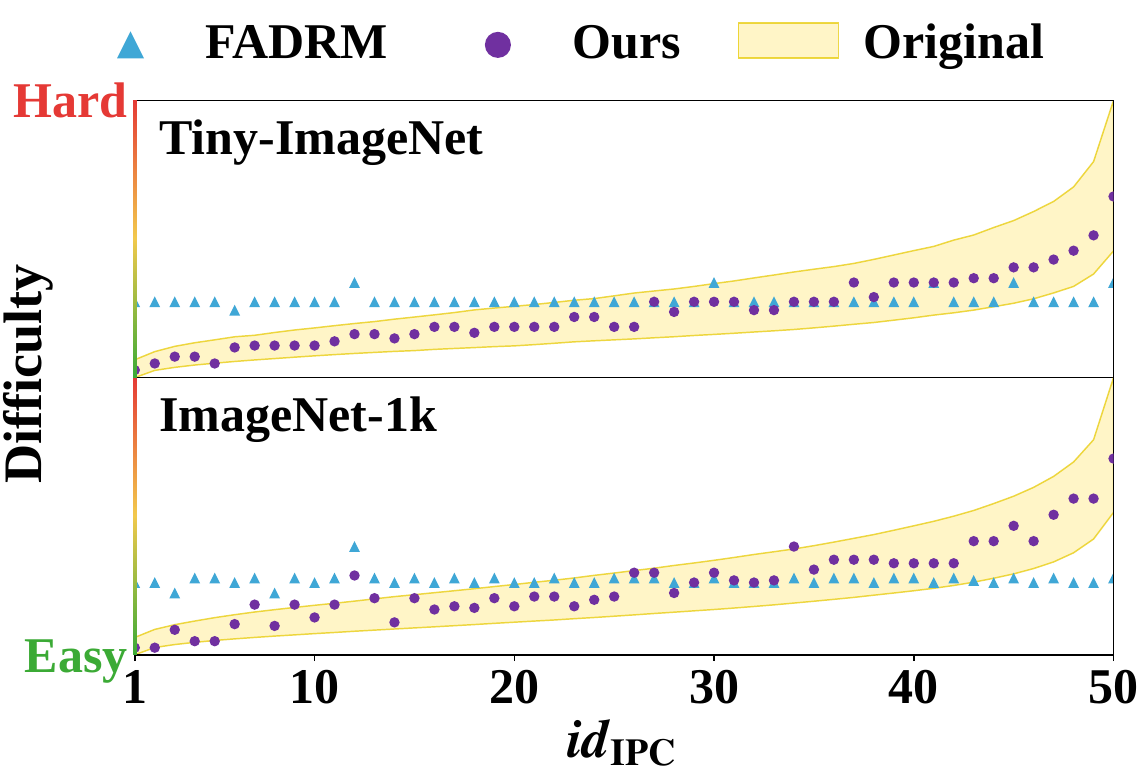}
        \caption{}
        \label{fig:difficulty_Tiny-ImageNet_ImageNet-1k}
    \end{subfigure}

    \begin{subfigure}[t]{\linewidth}
        \centering
        \includegraphics[width=0.5\linewidth]
        {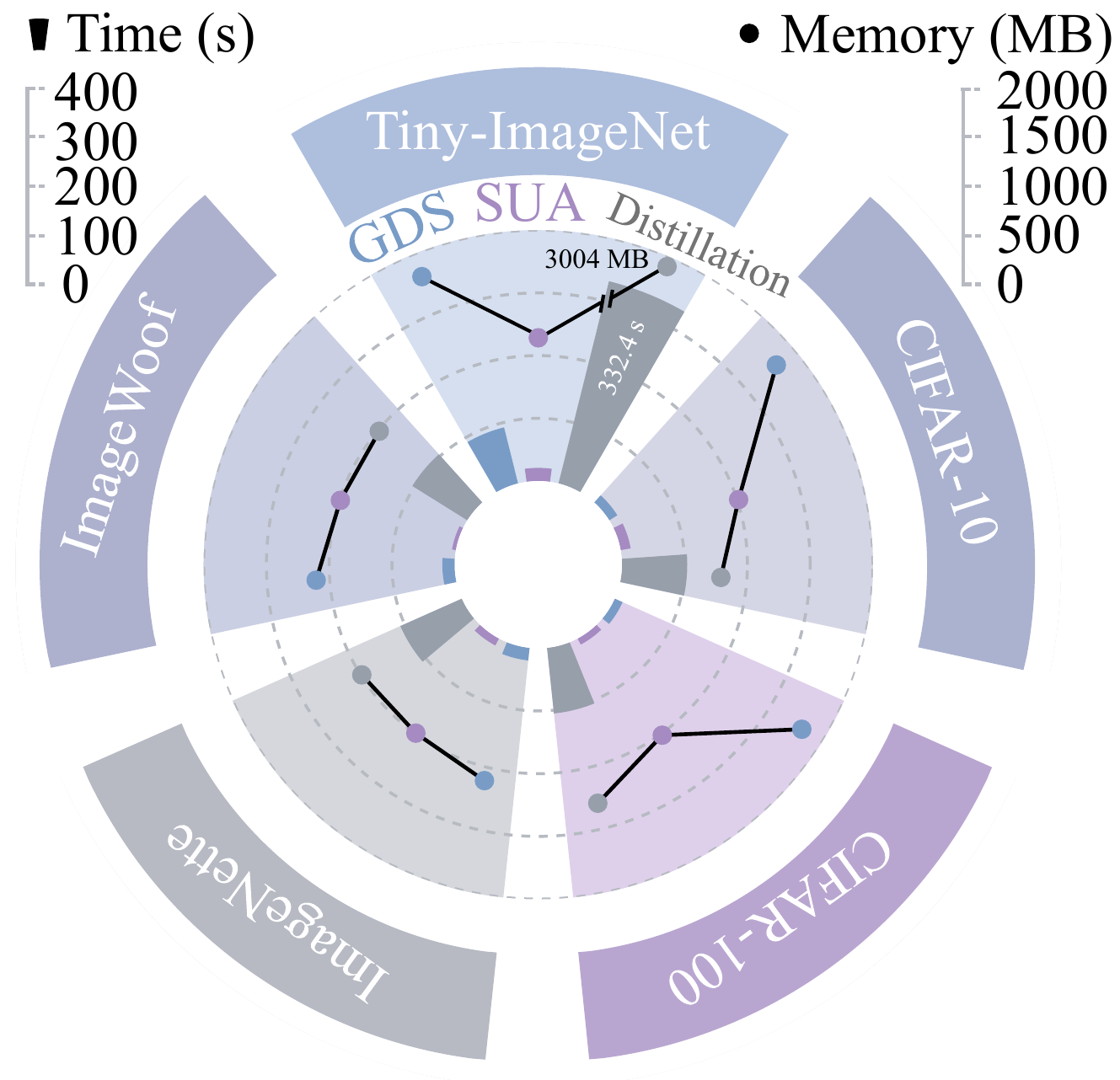}
        \caption{}
        \label{fig:efficiency}
    \end{subfigure}

    \caption{With ResNet18 as the teacher and $\mathrm{IPC}=50$, the difficulty distributions of the original data and the data distilled by FADRM and DSM on (a) CIFAR-10/100 and (b) Tiny-ImageNet/ImageNet-1k. (c) Runtime and peak GPU memory usage of GDS, as well as SUA and distillation for a single $id_{\mathrm{IPC}}$, with $\mathrm{IPC}=10$. ``//'' indicates that the corresponding value exceeds the plotting range.}
    \label{fig:difficulty_and_efficiency}
\end{figure}

\begin{table*}[t]
\centering
\caption{Cross-architecture Top-1 accuracy on small-scale datasets with $\mathrm{IPC}=10$. FADRM and DSM use ResNet18 as the teacher. Results evaluated with ResNet18 as the student are also reported.}
\label{tab:cross_architecture_small_scale}

\small
\setlength{\tabcolsep}{3pt}
\renewcommand{\arraystretch}{1.15}

\resizebox{\textwidth}{!}{%
\begin{tabular}{
@{}cc
*{4}{
    >{\columncolor{FADRMBlue}}c
    >{\columncolor{PSMPurple}}c
}
@{}
}
\toprule
\multicolumn{2}{c}{Datasets} & \multicolumn{2}{c}{CIFAR-10} & \multicolumn{2}{c}{CIFAR-100} & \multicolumn{2}{c}{ImageNette} & \multicolumn{2}{c}{ImageWoof} \\
\cmidrule(lr){1-2}
\cmidrule(lr){3-4}
\cmidrule(lr){5-6}
\cmidrule(lr){7-8}
\cmidrule(lr){9-10}

Student & Parameters & FADRM & DSM & FADRM & DSM & FADRM & DSM & FADRM & DSM \\
\midrule

ShuffleNetV2 & 2.3M & \tbscore{39.3}{1.0} & \tbscore{41.5}{0.8}\psmup{2.2} & \tbscore{57.2}{0.3} & \tbscore{58.6}{0.3}\psmup{1.4} & \tbscore{48.4}{1.1} & \tbscore{52.1}{0.2}\psmup{3.7} & \tbscore{30.5}{2.2} & \tbscore{28.9}{1.6}\psmdown{1.6} \\

MobileNetV2       & 3.4M & \tbscore{43.0}{0.4} & \tbscore{41.3}{0.5}\psmdown{1.7} & \tbscore{56.4}{0.2} & \tbscore{57.8}{0.3}\psmup{1.4} & \tbscore{54.2}{0.2} & \tbscore{56.1}{0.6}\psmup{1.9} & \tbscore{30.1}{1.0} & \tbscore{32.0}{1.4}\psmup{1.9} \\

DenseNet121       & 8.0M & \tbscore{48.0}{0.6} & \tbscore{46.0}{1.0}\psmdown{2.0} &\tbscore{62.1}{0.3} & \tbscore{63.9}{0.2}\psmup{1.8} & \tbscore{66.3}{0.6} & \tbscore{67.8}{1.0}\psmup{1.5} & \tbscore{43.0}{1.0} & \tbscore{43.6}{0.5}\psmup{0.6} \\

ResNet18          & 11.7M & \tbscore{48.2}{0.4} & \tbscore{54.5}{1.3}\psmup{6.3} & \tbscore{64.7}{0.4} & \tbscore{66.7}{0.4}\psmup{2.0} & \tbscore{64.0}{0.1} & \tbscore{67.1}{0.2}\psmup{3.1} & \tbscore{43.8}{0.2} & \tbscore{44.5}{0.5}\psmup{0.7} \\

ResNet50          & 25.6M & \tbscore{36.5}{1.0} & \tbscore{40.2}{1.3}\psmup{3.7} & \tbscore{60.2}{0.2} & \tbscore{62.6}{0.7}\psmup{2.4} & \tbscore{62.4}{0.7} & \tbscore{65.0}{1.0}\psmup{2.6} & \tbscore{40.5}{0.4} & \tbscore{40.5}{1.1}\psmeq{0.0} \\

ResNet101
                  & 44.5M & \tbscore{39.8}{0.1} & \tbscore{38.4}{1.0}\psmdown{1.4} & \tbscore{59.3}{1.0} & \tbscore{61.7}{0.4}\psmup{2.4} & \tbscore{55.7}{0.7} & \tbscore{62.2}{0.8}\psmup{6.5} & \tbscore{36.9}{0.7} & \tbscore{37.2}{0.2}\psmup{0.3} \\

\bottomrule
\end{tabular}%
}
\end{table*}

\subsubsection{Cross-architecture Generalization}

\noindent\textbf{Small-scale datasets.}
We evaluate cross-architecture generalization on CIFAR-10/100 and ImageNette/Woof using students with different scales and architectures. As shown in Table~\ref{tab:cross_architecture_small_scale}, under most settings, DSM consistently outperforms FADRM across all datasets and student architectures. For example, on ImageNette, DSM outperforms FADRM by 6.5\% when using ResNet101 (44.5M parameters) as the student, demonstrating its strong cross-architecture generalization ability.

\noindent\textbf{Large-scale datasets.}
We conduct cross-architecture generalization experiments on Tiny-ImageNet and ImageNet-1k, and further include DeiT-Ti (5.7M parameters), which adopts a Transformer architecture, to evaluate the generalization of DSM across different architectures. As shown in Table~\ref{tab:cross_architecture_large_scale}, DSM performs well with both ResNet architecture and the newer DeiT architecture, demonstrating strong cross-architecture generalization on large-scale datasets.

\begin{table}[!tb]
\centering

\caption{Cross-architecture Top-1 accuracy on large-scale datasets with $\mathrm{IPC}=50$. To further evaluate the generalization of the distilled data across different models, we also use DeiT~\citep{DeiT}, which adopts a Transformer~\citep{Transformer} architecture, as a student.}
\label{tab:cross_architecture_large_scale}

\small
\setlength{\tabcolsep}{2pt}
\renewcommand{\arraystretch}{1.15}

\resizebox{\linewidth}{!}{%
\begin{tabular}{
@{}cc
*{2}{
    >{\columncolor{FADRMBlue}}c
    >{\columncolor{PSMPurple}}c
}
@{}
}
\toprule
\multicolumn{2}{c}{Datasets}
& \multicolumn{2}{c}{Tiny-ImageNet}
& \multicolumn{2}{c}{ImageNet-1k} \\
\cmidrule(lr){1-2}
\cmidrule(lr){3-4}
\cmidrule(lr){5-6}

Student & Parameters & FADRM & DSM & FADRM & DSM \\
\midrule

DeiT-Ti & 5.7M
& \tbscore{18.6}{0.3}
& \tbscore{19.9}{0.2}\psmup{1.3}
& \tbscore{36.5}{0.5}
& \tbscore{37.3}{0.1}\psmup{0.8} \\

ResNet18 & 11.7M
& \tbscore{51.4}{0.1}
& \tbscore{51.7}{0.1}\psmup{0.3}
& \tbscore{60.9}{0.1}
& \tbscore{61.3}{0.1}\psmup{0.4} \\

ResNet101 & 44.5M
& \tbscore{51.1}{1.8}
& \tbscore{53.1}{0.9}\psmup{2.0}
& \tbscore{65.5}{0.1}
& \tbscore{65.2}{0.1}\psmdown{0.3} \\

\bottomrule
\end{tabular}%
}
\end{table}

\subsection{Ablation Study}

\noindent\textbf{Difficulty Calculation Methods.}
We compare GDS with existing difficulty calculation methods, including Forgetting Score~\citep{forgetting_score}, Confidence Score~\citep{confidence_score}, and Logits~\citep{logits}. As shown in Table~\ref{tab:difficulty_calculation}, GDS achieves the highest Top-1 accuracy on all datasets, indicating that it effectively partitions the original data into groups with distinct difficulties, provides reliable difficulty groups for SUA, and validates the effectiveness of DSM.

\begin{table}[!tb]
\centering

\caption{Top-1 accuracy of different difficulty calculation methods with $\mathrm{IPC}=10$ and ResNet18 as both the teacher and student.}
\label{tab:difficulty_calculation}

\small
\setlength{\tabcolsep}{4pt}
\renewcommand{\arraystretch}{1.15}

\resizebox{\linewidth}{!}{%
\begin{tabular}{@{}ccccc@{}}
\toprule
Difficulty Calculation & CIFAR-10 & CIFAR-100 & ImageNette & ImageWoof \\
\midrule

Forgetting score
& \tbscore{\underline{48.5}}{0.2}
& \tbscore{65.3}{0.2}
& \tbscore{\underline{65.6}}{0.3}
& \tbscore{39.6}{1.1} \\

Confidence score
& \tbscore{43.3}{0.5}
& \tbscore{\underline{65.7}}{0.6}
& \tbscore{64.0}{0.4}
& \tbscore{42.1}{0.6} \\

Logits
& \tbscore{43.2}{1.1}
& \tbscore{\underline{65.7}}{0.3}
& \tbscore{64.8}{0.9}
& \tbscore{\underline{43.8}}{0.6} \\

GDS
& \tbscore{\textbf{54.5}}{1.3}
& \tbscore{\textbf{66.7}}{0.4}
& \tbscore{\textbf{67.1}}{0.2}
& \tbscore{\textbf{44.5}}{0.5} \\

\bottomrule
\end{tabular}%
}
\end{table}

\begin{table}[!tb]
\centering

\caption{Top-1 accuracy in the ablation study of SUA and ISS with $\mathrm{IPC}=10$ and ResNet18 as both the teacher and student.}
\label{tab:sua_iss_ablation}

\small
\setlength{\tabcolsep}{4pt}
\renewcommand{\arraystretch}{1.15}

\resizebox{\linewidth}{!}{%
\begin{tabular}{@{}cccccc@{}}
\toprule
SUA & ISS & CIFAR-10 & CIFAR-100 & ImageNette & ImageWoof \\
\midrule

{} & {}
& \tbscore{48.2}{0.4}
& \tbscore{64.7}{0.4}
& \tbscore{64.0}{0.1}
& \tbscore{\underline{43.8}}{0.2} \\

$\checkmark$ & {}
& \tbscore{\underline{50.7}}{0.5}
& \tbscore{65.2}{0.4}
& \tbscore{\underline{65.9}}{0.9}
& \tbscore{42.1}{1.3} \\

{} & $\checkmark$
& \tbscore{47.6}{0.1}
& \tbscore{\underline{65.3}}{0.2}
& \tbscore{65.8}{0.4}
& \tbscore{40.0}{0.8} \\

$\checkmark$ & $\checkmark$
& \tbscore{\textbf{54.5}}{1.3}
& \tbscore{\textbf{66.7}}{0.4}
& \tbscore{\textbf{67.1}}{0.2}
& \tbscore{\textbf{44.5}}{0.5} \\

\bottomrule
\end{tabular}%
}
\end{table}

\begin{figure*}[t]
    \centering
    \captionsetup{skip=5pt}
    \includegraphics[width=0.9\textwidth]{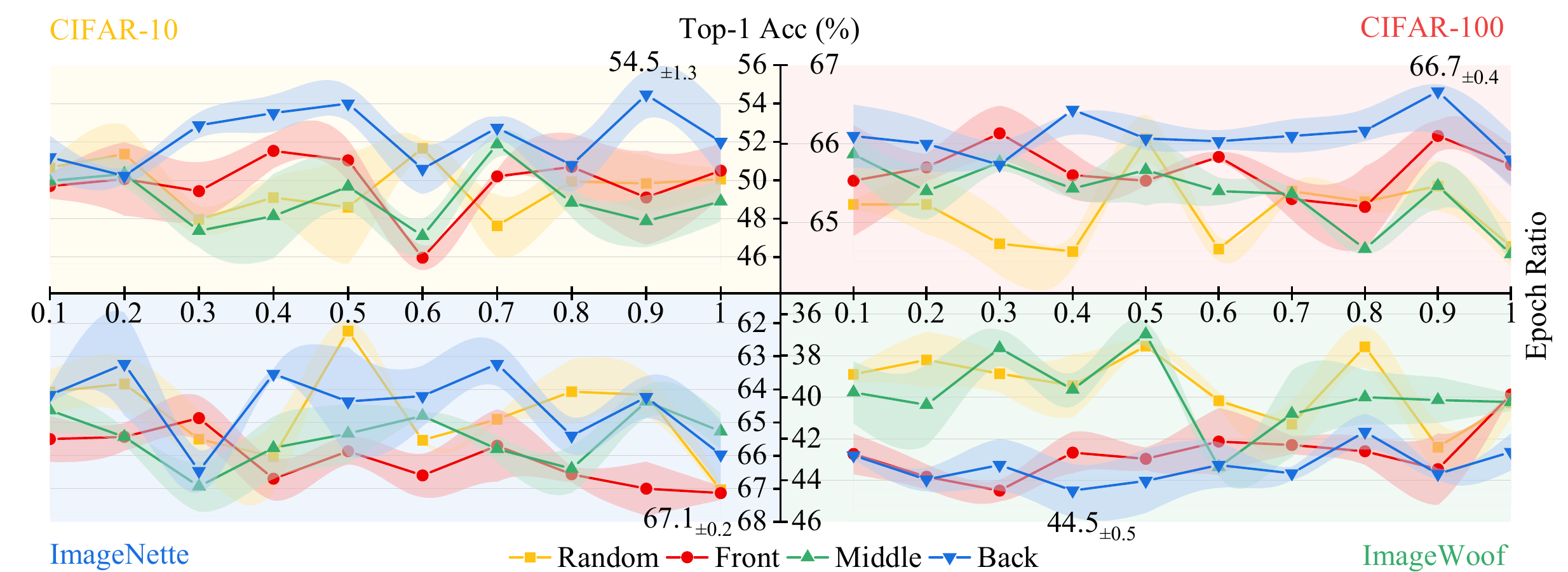}
    \caption{Hyperparameter study of SUA and ISS. On CIFAR-10/100 and ImageNette/Woof, we set $\mathrm{IPC}=10$ and use ResNet18 as both the teacher and student. We investigate the forward epoch ratio $r_{\mathrm{fwd}} \in \{0.1, 0.2, \ldots, 1.0\}$ in SUA and the screening method (Random, Front, Middle, Back) in ISS.}
    \label{fig:EpochRatio_SelectMethod}
\end{figure*}

\begin{figure*}[t]
    \centering
    \captionsetup{skip=5pt}
    \includegraphics[width=0.88\textwidth]{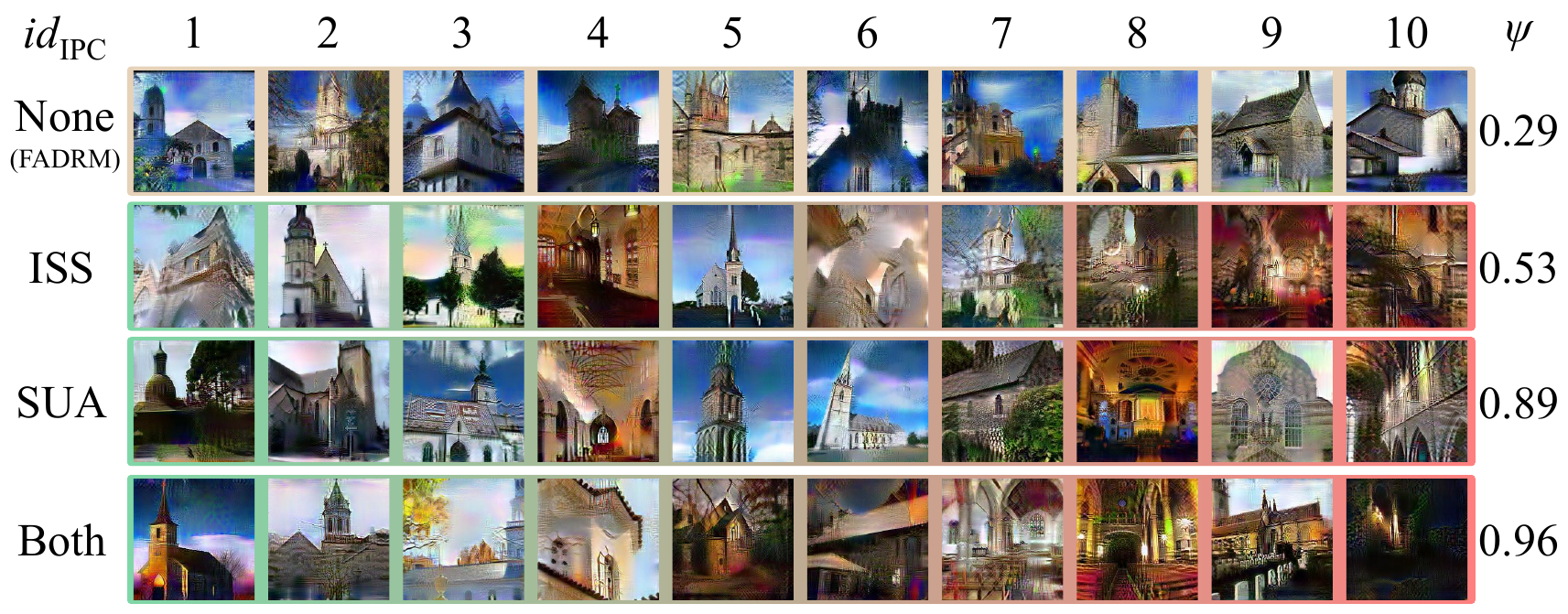}
    \caption{Distilled data on ImageNette with $\mathrm{IPC}=10$, using neither mechanism (FADRM), ISS, SUA, or both. Spearman's $\psi$ measures the agreement between the difficulty trends of the distilled and original data, with values closer to 1 indicating better agreement.}
    \label{fig:Visualization}
\end{figure*}

\noindent\textbf{Epoch Ratio in SUA and Screening Method in ISS.} SUA updates the teacher BN statistics through forward passes. We multiply the epoch ratio $r_{\mathrm{fwd}} \in \{0.1, 0.2, \ldots, 1.0\}$ by the pretraining epochs $E_k^{\mathrm{pre}}$ of teacher $k$ to determine the number of SUA forward epochs. For ISS, we compare four screening methods within each difficulty group: Random samples randomly, while Front, Middle, and Back select images from the front, middle, and back of the difficulty ranking, respectively.

We systematically study $r_{\mathrm{fwd}}$ and ISS methods on CIFAR-10/100 and ImageNette/Woof. As shown in Figure~\ref{fig:EpochRatio_SelectMethod}, Back performs best on CIFAR-10/100 and ImageWoof, while Front performs best on ImageNette. Under the optimal methods, CIFAR-10/100 (54.5\% / 66.7\%) and ImageNette (67.1\%) achieve the highest accuracy at $r_{\mathrm{fwd}}=0.9/1.0$, whereas ImageWoof (44.5\%) performs best at $r_{\mathrm{fwd}}=0.4$. Too few forward passes may update statistics insufficiently, while too many may overadapt them to the current difficulty group. Thus, the optimal $r_{\mathrm{fwd}}$ depends on the dataset distribution.

\noindent\textbf{Ablation of Mechanisms.}
SUA and ISS control distilled data difficulty through supervision and initialization, respectively. As shown in Table~\ref{tab:sua_iss_ablation}, removing both reduces DSM to FADRM. Individually, they improve several datasets but hurt ImageWoof, likely due to mismatched initialization and supervision. Jointly, they achieve the best results across all datasets, demonstrating their complementarity in constructing the difficulty structure.

During SUA, we use the same batch size as pretraining. For DSM+, all teachers share the same $r_{\mathrm{fwd}}$ and pretraining batch size. Further details and ablations are provided in the supplementary material.

\subsection{Visualization}
As shown in Figure~\ref{fig:Visualization}, we use the Spearman rank correlation coefficient $\psi$~\citep{spearman} to quantitatively measure the alignment between the difficulty of distilled samples and that of the original data, where a higher value indicates better alignment. Without either mechanism (FADRM), the samples exhibit nearly identical difficulty and achieve the lowest $\psi$. Using either ISS or SUA introduces clear difficulty variations, while combining both produces more pronounced variations and achieves the highest $\psi$. These results demonstrate the effectiveness of DSM in improving the difficulty structure of distilled samples.

\section{Conclusion}
We have proposed DSM, which uses GDS to rank the original samples by difficulty within each class and divide them into groups. For each group, SUA updates the teacher's BN running statistics through forward passes, while ISS selects real samples from the same group to initialize the corresponding distilled batch. Compared with FADRM, DSM better preserves the range of difficulties in the original data and improves student accuracy in most settings.

{
    \small
    \bibliographystyle{ieeenat_fullname}
    \bibliography{main}
}

\twocolumn[{%
  \centering
  {\Large\bfseries Supplementary Material for\par
Difficulty-Aware Statistical Matching for Dataset Distillation\par}
  \vspace{1em}
}]
\appendix

\setcounter{table}{0}
\setcounter{figure}{0}
\setcounter{equation}{0}
\renewcommand{\thetable}{S\arabic{table}}
\renewcommand{\thefigure}{S\arabic{figure}}
\renewcommand{\theequation}{S\arabic{equation}}
\renewcommand{\theHtable}{supp.\arabic{table}}
\renewcommand{\theHfigure}{supp.\arabic{figure}}
\renewcommand{\theHequation}{supp.\arabic{equation}}

\makeatletter
\setlength{\@fptop}{0pt}
\setlength{\@fpsep}{12pt plus 2pt minus 2pt}
\setlength{\@fpbot}{0pt plus 1fil}
\setlength{\@dblfptop}{0pt}
\setlength{\@dblfpsep}{12pt plus 2pt minus 2pt}
\setlength{\@dblfpbot}{0pt plus 1fil}
\makeatother

\section{Additional Related Work}
\label{additional related work}
\paragraph{Dataset Distillation.}
Dataset distillation (DD) compresses a large dataset into a small set of samples with high training utility, substantially reducing storage and training costs while maintaining performance comparable to training on the full dataset. Existing methods mainly include gradient matching~\citep{DC,DSA,DCC,DREAM}, distribution matching~\citep{DM,CAFE,DataDAM,GUARD,li2025davdd}, trajectory matching~\citep{TESLA,APM,DATM}, decoupled distillation~\citep{SRe2L,FADRM,EDC,CDA,LPLD}, and generative distillation~\citep{HPD,Minimax,D4M,ye2025igds,cai2026evlf}.

\paragraph{Curriculum Dataset Distillation.}
CDA~\citep{CDA} gradually adjusts crop sizes during synthesis, shifting optimization from global structure to local details. CUDD~\citep{CUDD} uses student misclassifications to identify samples poorly represented by the current distilled set and synthesizes progressively harder images. At each stage, CCFS~\citep{logits} expands the distilled set with relatively easy real images that the current model misclassifies. SelMatch~\citep{SelMatch} selects real samples by difficulty and applies trajectory matching to only a subset, preserving rare features from harder samples. Overall, easy samples capture common class features, while harder samples add complex features that might otherwise be missed, improving the training utility of the distilled set.

\paragraph{Decoupled Dataset Distillation.}
SRe$^2$L~\citep{SRe2L} pioneered decoupled DD by separating teacher pretraining from distilled data optimization. SRe$^2$L++~\citep{CVDD} improves robustness through stronger data augmentation and soft labels specific to each batch. FADRM~\citep{FADRM} uses multiscale data residual connections to preserve original information and enrich distilled samples, while FADRM+ extends it to multiple teachers. G-VBSM~\citep{GVBSM} employs lightweight model ensembles to improve generalization across architectures, while CDA~\citep{CDA} stabilizes optimization through a curriculum strategy. LPLD~\citep{LPLD} reexamines the need for large volumes of soft labels and provides a lighter alternative. Despite their strong performance on standard benchmarks, these methods do not explicitly consider the difficulty structure of the original data and its influence on distilled data.

\paragraph{Batch Normalization.}
Batch Normalization (BN)~\citep{BN} normalizes intermediate features using batch means and variances, improving training speed and stability. Batch Renormalization~\citep{BatchRenorm} uses running statistics to reduce dependence on batch composition and improve training with small batches. Santurkar et al.~\citep{BNOptimization} show that BN smooths the optimization landscape and stabilizes gradients. Luo et al.~\citep{BNRegularization} interpret BN as implicit regularization and analyze its effects on convergence and generalization. AdaBN~\citep{AdaBN} adapts BN statistics across domains, showing that they characterize domain distributions. Overall, BN improves training stability, while its running means and variances capture layer-wise feature distributions, supporting their use as statistical supervision in decoupled dataset distillation.

\section{Theoretical Analysis}

\subsection{The Observed Relationship Between GDS and Classification Difficulty}
\label{empirial_relationship}
GDS measures the distance between the statistics produced by a sample at the inputs to the teacher's BN layers and the corresponding BN running statistics, thereby characterizing its deviation from the feature distribution learned by the teacher. To analyze the correlation between statistical deviation and classification difficulty, we measure sample difficulty using the teacher's cross-entropy loss on the ground truth label. For a sample $(x_i,y_i)$, its classification difficulty is defined as
\begin{equation}
\ell_i
=
\mathcal{L}_{\mathrm{CE}}\!\left(f_{\theta}(x_i),y_i\right)
=
-\log p_{\theta}(y_i\mid x_i),
\end{equation}
where $p_{\theta}(y_i\mid x_i)$ denotes the probability that the teacher assigns to the ground truth class $y_i$. A larger $\ell_i$ indicates lower confidence in the ground truth class and thus greater classification difficulty.

We conduct this analysis on CIFAR-10/100 and ImageNette/Woof using a pretrained ResNet18 teacher to compute the GDS of each original sample. Suppose the dataset contains $N$ samples, and let $\pi$ denote the index sequence obtained by sorting them in ascending order of GDS:
\begin{equation}
\operatorname{GDS}\!\left(x_{\pi(1)}\right)
\leq
\operatorname{GDS}\!\left(x_{\pi(2)}\right)
\leq
\cdots
\leq
\operatorname{GDS}\!\left(x_{\pi(N)}\right).
\end{equation}
We then compute the corresponding cross-entropy loss $\ell_{\pi(r)}$ for each ordered sample and examine how the classification loss changes along the GDS ranking. As shown in Figure~\ref{fig:gps_difficulty}, we rank the samples of each class by GDS and compute their cross-entropy losses, then average the results across classes. We observe that the cross-entropy loss generally increases with GDS, with larger GDS values corresponding to higher classification losses. This indicates that the statistical deviations measured by GDS follow a similar trend to the classification difficulty perceived by the teacher, supporting the use of GDS to rank and group original samples from easy to hard.

\begin{figure*}[t]
    \centering

    \begin{subfigure}[t]{0.49\textwidth}
        \centering
        \includegraphics[width=\linewidth]
        {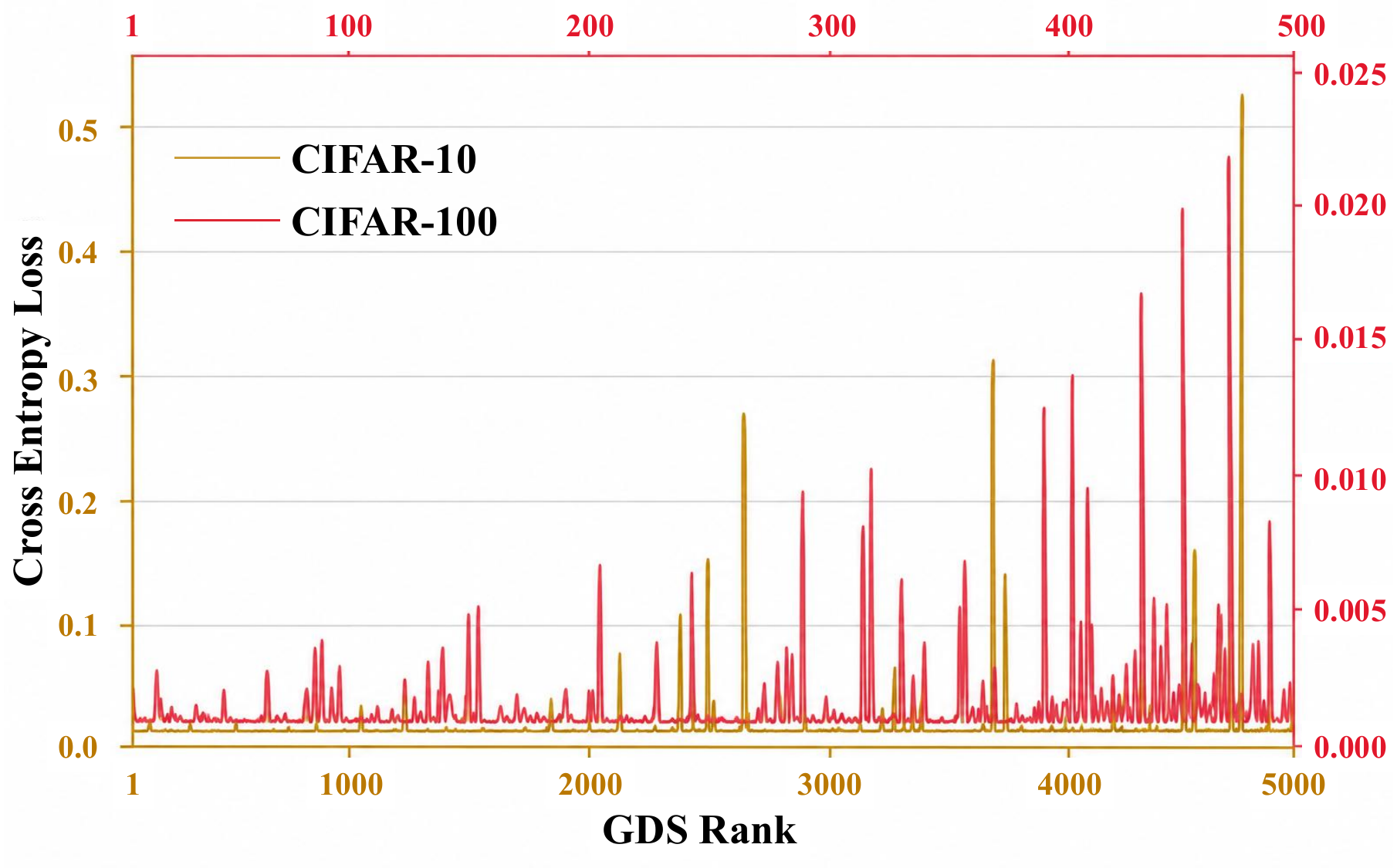}
        \caption{CIFAR-10 and CIFAR-100.}
        \label{fig:gps_cifar}
    \end{subfigure}
    \hfill
    \begin{subfigure}[t]{0.49\textwidth}
        \centering
        \includegraphics[width=\linewidth]
        {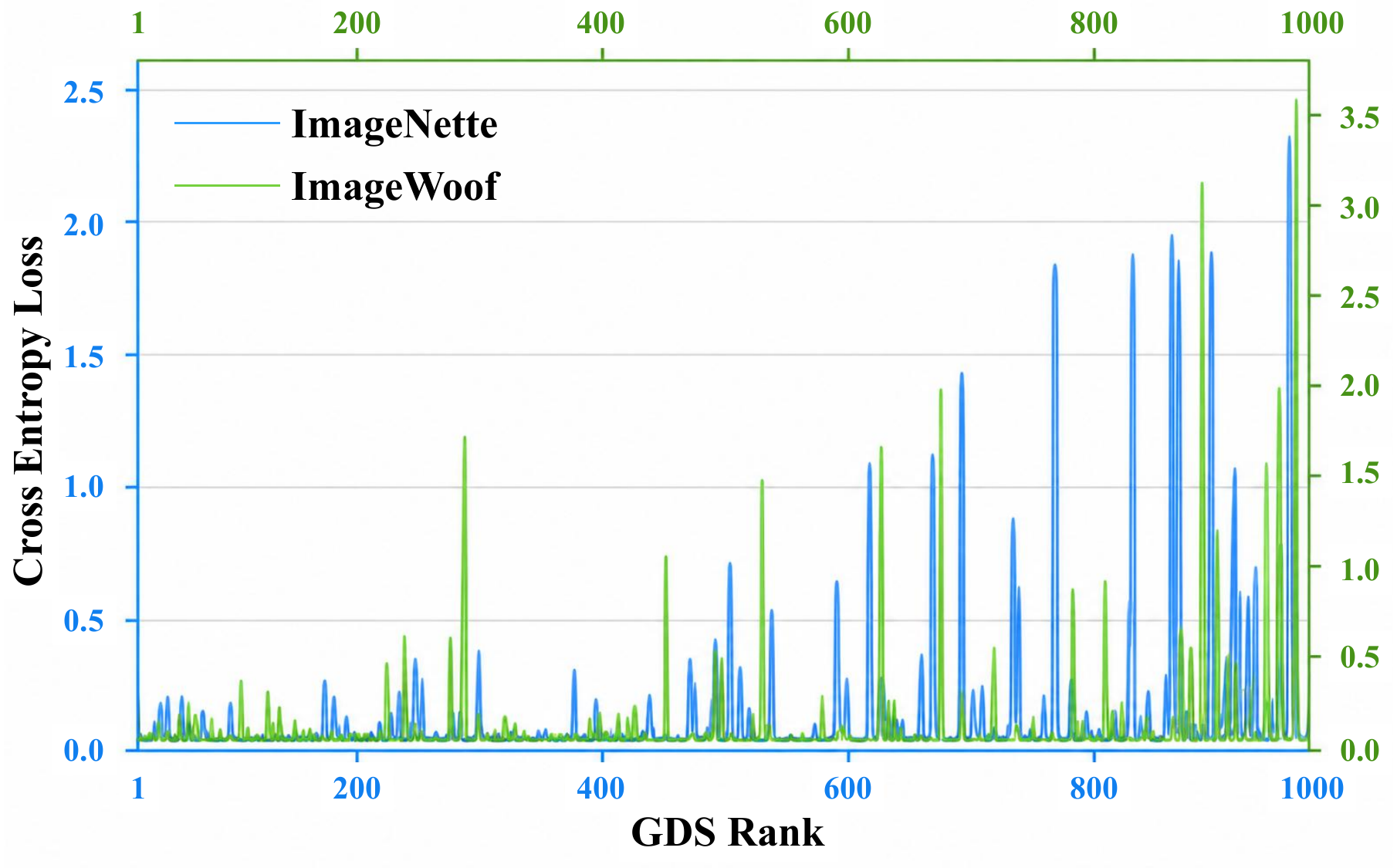}
        \caption{ImageNette and ImageWoof.}
        \label{fig:gps_imagenette_imagewoof}
    \end{subfigure}

    \caption{Relationship between GDS and classification difficulty on different datasets.}
    \label{fig:gps_difficulty}
\end{figure*}

\subsection{The Expected Convergence of Statistics}
\label{app:the expected convergence of statistics}
\begin{proof}[Proof of Proposition~\ref{prop:expected_convergence}]
Fix an arbitrary teacher $k$ and difficulty group $g$. Let $\mathcal{B}_t$ denote the $t$-th forward batch independently drawn from $\mathcal{T}_g$, and let $\widehat{\boldsymbol{\tau}}_{k,g}^{(t)}$ denote the batch statistics induced by $\mathcal{B}_t$ at all BN layers. Since the teacher parameters are fixed and the forward batches are identically distributed, we have
\begin{equation}
\mathbb{E}
\left[
\widehat{\boldsymbol{\tau}}_{k,g}^{(t)}
\right]
=
\boldsymbol{\tau}_{k,g}^{\star},
\qquad \forall t.
\label{eq:expected_batch_statistics}
\end{equation}

Applying the BN update rule in~\eqref{eq:bn_running_statistics} componentwise to all running means and variances gives
\begin{equation}
\boldsymbol{\tau}_{k,g}^{(t)}
=
(1-\rho)
\boldsymbol{\tau}_{k,g}^{(t-1)}
+
\rho
\widehat{\boldsymbol{\tau}}_{k,g}^{(t)},
\label{eq:sua_statistics_update}
\end{equation}
where SUA starts from the global statistics stored in the pretrained teacher:
\begin{equation}
\boldsymbol{\tau}_{k,g}^{(0)}
=
\boldsymbol{\tau}_{k}^{\mathrm G}.
\label{eq:sua_initial_statistics}
\end{equation}

Taking expectations in~\eqref{eq:sua_statistics_update} and using~\eqref{eq:expected_batch_statistics}, we obtain
\begin{align}
\mathbb{E}
\left[
\boldsymbol{\tau}_{k,g}^{(t)}
\right]
-
\boldsymbol{\tau}_{k,g}^{\star}
&=
(1-\rho)
\left(
\mathbb{E}
\left[
\boldsymbol{\tau}_{k,g}^{(t-1)}
\right]
-
\boldsymbol{\tau}_{k,g}^{\star}
\right).
\label{eq:sua_expected_error}
\end{align}
Recursively applying~\eqref{eq:sua_expected_error} and substituting~\eqref{eq:sua_initial_statistics} yields
\begin{equation}
\mathbb{E}
\left[
\boldsymbol{\tau}_{k,g}^{(t)}
\right]
-
\boldsymbol{\tau}_{k,g}^{\star}
=
(1-\rho)^t
\left(
\boldsymbol{\tau}_{k}^{\mathrm G}
-
\boldsymbol{\tau}_{k,g}^{\star}
\right).
\label{eq:sua_expected_convergence}
\end{equation}

Since $\rho\in(0,1]$, we have $0\leq 1-\rho<1$. Therefore,
\begin{equation}
\mathbb{E}
\left[
\boldsymbol{\tau}_{k,g}^{(t)}
\right]
-
\boldsymbol{\tau}_{k,g}^{\star}
\xrightarrow{t\rightarrow\infty}
\boldsymbol{0}.
\end{equation}
Thus, the expected discrepancy between the SUA running statistics and the statistics of the current difficulty group decreases geometrically with the number of forward batch updates. This result holds for the running means and variances of all BN layers. Consequently, SUA statistics converge in expectation to those of the current difficulty group and provide the corresponding statistical matching signal for distillation batch $g$.
\end{proof}

\section{More Results}
\label{more results}
\subsection{Comparison with other SOTA Methods}
To further evaluate the effectiveness, we compare DSM with more representative SOTA dataset distillation methods, including the coreset selection methods Random and Herding~\citep{Herding}, the gradient matching method DSA~\citep{DSA}, the distribution matching method DM~\citep{DM}, the trajectory matching methods MTT~\citep{MTT}, DATM~\citep{DATM} and TESLA~\citep{TESLA}, the generative method Minimax~\citep{Minimax}, and the decoupled methods SRe$^2$L and FADRM. Specifically, DSA uses ConvNet as the teacher model, DM uses ConvNet and ResNetAP-10, MTT uses ConvNet-BN, DATM and TESLA use ConvNet-IN, and Minimax uses DiT~\citep{DiT}, while SRe$^2$L, FADRM, and DSM use ResNet18. All methods are evaluated using ResNet18 as the student model.

As shown in Table~\ref{tab:other_sota_results}, coreset selection methods generally achieve lower accuracy, but gradually approach DD methods as IPC increases. For example, Random achieves $75.8\%$ on ImageNette with $\mathrm{IPC}=50$. DD methods perform better in most settings, where gradient matching, distribution matching, trajectory matching, and generative methods use gradients, feature distributions, weight trajectories, and generative priors to guide distillation, respectively. Although these methods compress data effectively, their supervision signals usually describe the overall data distribution in an aggregated form, making it difficult to capture the internal difficulty structure. As dataset scale increases, sample difficulty differences are more easily smoothed by aggregated objectives, causing distilled data to concentrate at similar difficulties and limiting coverage of the original difficulty range. Among decoupled methods, FADRM performs strongly, while DSM preserves a broader difficulty structure, enabling students to learn sample patterns ranging from easy to difficult, and achieving the best results in most settings.

\begin{table*}[t]
\centering
\begingroup
\providecommand{\tbscore}[2]{\ensuremath{#1_{\pm #2}}}
\providecolor{FADRMBlue}{RGB}{226,239,249}
\providecolor{PSMPurple}{RGB}{241,233,248}

\caption{Top-1 accuracy (\%) comparison of representative dataset distillation methods with ResNet18 as the student. Standard deviations are reported when available.}
\label{tab:other_sota_results}

\small
\setlength{\tabcolsep}{3pt}
\renewcommand{\arraystretch}{1.5}

\resizebox{\textwidth}{!}{%
\begin{tabular}{@{}lc*{9}{c}
    >{\columncolor{FADRMBlue}}c
    >{\columncolor{PSMPurple}}c@{}}
\toprule
\multirow{2}{*}{Dataset}
& \multirow{2}{*}{IPC}

& \multicolumn{2}{c}{Coreset} & Gradient & Distribution & \multicolumn{3}{c}{Trajectory} & Generative & \multicolumn{3}{c}{Decoupled} \\
\cmidrule(lr){3-4}
\cmidrule(lr){5-5}
\cmidrule(lr){6-6}
\cmidrule(lr){7-9}
\cmidrule(lr){10-10}
\cmidrule(lr){11-13}
& & Random & Herding & DSA & DM & MTT & DATM & TESLA & Minimax & SRe$^2$L & FADRM & DSM \\
\midrule

\multirow{3}{*}{CIFAR-10}
& 1 & -- & -- & -- & -- & -- & -- & -- & -- & \tbscore{16.6}{0.9} & \tbscore{\underline{19.3}}{0.6} & \tbscore{\textbf{20.9}}{0.1} \\

& 10 & \tbscore{25.1}{0.5} & \tbscore{28.4}{0.1} & \tbscore{42.1}{0.6} & \tbscore{38.2}{1.1} & \tbscore{46.1}{1.4} & 48.66 & \tbscore{\underline{48.9}}{2.2} & -- & \tbscore{29.3}{0.5} & \tbscore{48.2}{0.4} & \tbscore{\textbf{54.5}}{1.3} \\

& 50 & 54.96 & -- & \tbscore{47.8}{0.9} & \tbscore{52.9}{0.4} & \tbscore{58.7}{0.2} & 66.27 & -- & -- & \tbscore{45.0}{0.7} & \tbscore{\underline{80.6}}{0.7} & \tbscore{\textbf{81.1}}{0.1} \\
\midrule

\multirow{3}{*}{CIFAR-100}
& 1 & -- & -- & -- & -- & -- & -- & -- & -- & \tbscore{6.6}{0.2} & \tbscore{\underline{31.3}}{0.2} & \tbscore{\textbf{32.9}}{0.9} \\

& 10 & \tbscore{10.9}{0.1} & \tbscore{13.3}{0.3} & \tbscore{21.9}{0.4} & \tbscore{18.7}{0.5} & \tbscore{26.8}{0.6} & -- & \tbscore{27.1}{0.7} & -- & \tbscore{27.0}{0.4} & \tbscore{\underline{64.7}}{0.4} & \tbscore{\textbf{66.7}}{0.4} \\

& 50 & \tbscore{40.7}{1.0} & -- & \tbscore{43.6}{0.7} & \tbscore{42.6}{0.5} & \tbscore{51.3}{0.4} & \tbscore{51.0}{0.5} & -- & -- & \tbscore{50.2}{0.4} & \tbscore{\underline{69.3}}{0.3} & \tbscore{\textbf{69.7}}{0.3} \\
\midrule

\multirow{3}{*}{ImageNette}
& 1 & -- & -- & -- & -- & -- & -- & -- & -- & \tbscore{19.1}{1.1} & \tbscore{\textbf{28.2}}{0.6} & \tbscore{\underline{25.0}}{0.5} \\

& 10 & \tbscore{55.8}{1.0} & -- & -- & \tbscore{60.9}{0.7} & -- & -- & -- & \tbscore{\underline{64.9}}{0.6} & \tbscore{29.4}{3.0} & \tbscore{64.0}{0.1} & \tbscore{\textbf{67.1}}{0.2} \\

& 50 & \tbscore{75.8}{1.1} & -- & -- & \tbscore{75.0}{1.0} & -- & -- & -- & \tbscore{78.1}{0.6} & \tbscore{40.9}{0.3} & \tbscore{\underline{82.8}}{0.6} & \tbscore{\textbf{84.0}}{0.3} \\
\midrule

\multirow{3}{*}{ImageWoof}
& 1 & -- & -- & -- & -- & -- & -- & -- & -- & \tbscore{13.3}{0.5} & \tbscore{\underline{19.3}}{0.3} & \tbscore{\textbf{20.2}}{0.7} \\

& 10 & \tbscore{27.7}{0.9} & \tbscore{30.2}{1.2} & -- & \tbscore{33.4}{0.7} & -- & -- & -- & \tbscore{37.6}{0.9} & \tbscore{20.2}{0.2} & \tbscore{\underline{43.8}}{0.2} & \tbscore{\textbf{44.5}}{0.5} \\

& 50 & \tbscore{47.9}{1.8} & \tbscore{48.3}{1.2} & -- & \tbscore{46.2}{0.6} & -- & -- & -- & \tbscore{57.1}{0.6} & \tbscore{23.3}{0.3} & \tbscore{\underline{67.7}}{1.0} & \tbscore{\textbf{69.6}}{0.9} \\
\midrule

\multirow{3}{*}{Tiny-ImageNet} & 1 & -- & -- & -- & -- & -- & -- & -- & -- & \tbscore{2.6}{0.1} & \tbscore{\underline{28.6}}{0.1} & \tbscore{\textbf{30.6}}{1.3} \\

& 10 & \tbscore{7.5}{0.1} & \tbscore{9.0}{0.3} & -- & -- & -- & -- & -- & -- & \tbscore{16.1}{0.2} & \tbscore{\underline{46.5}}{0.4} & \tbscore{\textbf{48.1}}{0.3} \\

& 50 & \tbscore{30.1}{0.6} & -- & \tbscore{27.8}{1.4} & \tbscore{31.0}{0.6} & \tbscore{40.3}{0.3} & \tbscore{42.2}{0.2} & -- & -- & \tbscore{41.1}{0.4} & \tbscore{\underline{51.4}}{0.1} & \tbscore{\textbf{51.7}}{0.1} \\
\midrule

\multirow{3}{*}{ImageNet-1K}
& 1  & -- & -- & -- & -- & -- & -- & -- & --
     & \tbscore{0.1}{0.1} & \tbscore{\underline{7.9}}{0.4} & -- \\

& 10 & \tbscore{4.4}{0.1} & \tbscore{5.8}{0.1} & -- & -- & -- & -- & \tbscore{7.7}{0.1} & \tbscore{44.3}{0.5} & \tbscore{21.3}{0.6} & \tbscore{\underline{47.8}}{0.4} & \tbscore{\textbf{48.8}}{0.2} \\

& 50 & -- & -- & -- & -- & -- & -- & -- & \tbscore{58.6}{0.3} & \tbscore{46.8}{0.2} & \tbscore{\underline{60.9}}{0.1} & \tbscore{\textbf{61.3}}{0.1} \\

\bottomrule
\end{tabular}%
}
\endgroup
\end{table*}

\begin{table}[!htbp]
\centering

\caption{Top-1 accuracy under different epoch ratios on ImageNette with $\mathrm{IPC}=10$ under the same-architecture protocol.}
\label{tab:epoch_ratio_models}

\small
\setlength{\tabcolsep}{3pt}
\renewcommand{\arraystretch}{1.5}

\resizebox{\linewidth}{!}{%
\begin{tabular}{@{}ccccc@{}}
\toprule
Epoch Ratio & DenseNet121 & MobileNetV2 & ShuffleNetV2 & ResNet18 \\
\midrule

0.1 & \tbscore{68.6}{0.6} & \tbscore{51.5}{0.6} & \tbscore{49.1}{1.1} & \tbscore{65.5}{0.7} \\

0.2 & \tbscore{68.8}{0.7} & \tbscore{51.8}{1.0} & \tbscore{\underline{50.2}}{2.5} & \tbscore{65.4}{0.4} \\

0.3 & \tbscore{68.9}{1.0} & \tbscore{50.0}{0.7} & \tbscore{48.8}{0.1} & \tbscore{64.9}{0.7} \\

0.4 & \tbscore{68.3}{0.9} & \tbscore{53.7}{1.1} & \tbscore{48.0}{2.1} & \tbscore{66.7}{0.7} \\

0.5 & \tbscore{67.4}{1.1} & \tbscore{51.1}{0.5} & \tbscore{\textbf{50.3}}{1.3} & \tbscore{65.9}{0.4} \\

0.6 & \tbscore{67.2}{0.2} & \tbscore{52.3}{0.8} & \tbscore{50.0}{1.5} & \tbscore{66.6}{0.7} \\

0.7 & \tbscore{67.7}{1.2} & \tbscore{\textbf{54.2}}{0.7} & \tbscore{50.1}{1.6} & \tbscore{65.7}{1.1} \\

0.8 & \tbscore{68.9}{0.3} & \tbscore{52.8}{1.5} & \tbscore{49.1}{0.8} & \tbscore{66.6}{0.2} \\

0.9 & \tbscore{\underline{69.1}}{0.6} & \tbscore{\underline{53.9}}{3.4} & \tbscore{49.4}{1.3} & \tbscore{\underline{67.0}}{0.8} \\

1.0 & \tbscore{\textbf{69.1}}{0.7} & \tbscore{52.0}{0.8} & \tbscore{49.5}{0.2} & \tbscore{\textbf{67.1}}{0.2} \\

\bottomrule
\end{tabular}%
}
\end{table}

\begin{figure}[!htbp]
\centering

\includegraphics[width=\linewidth]
{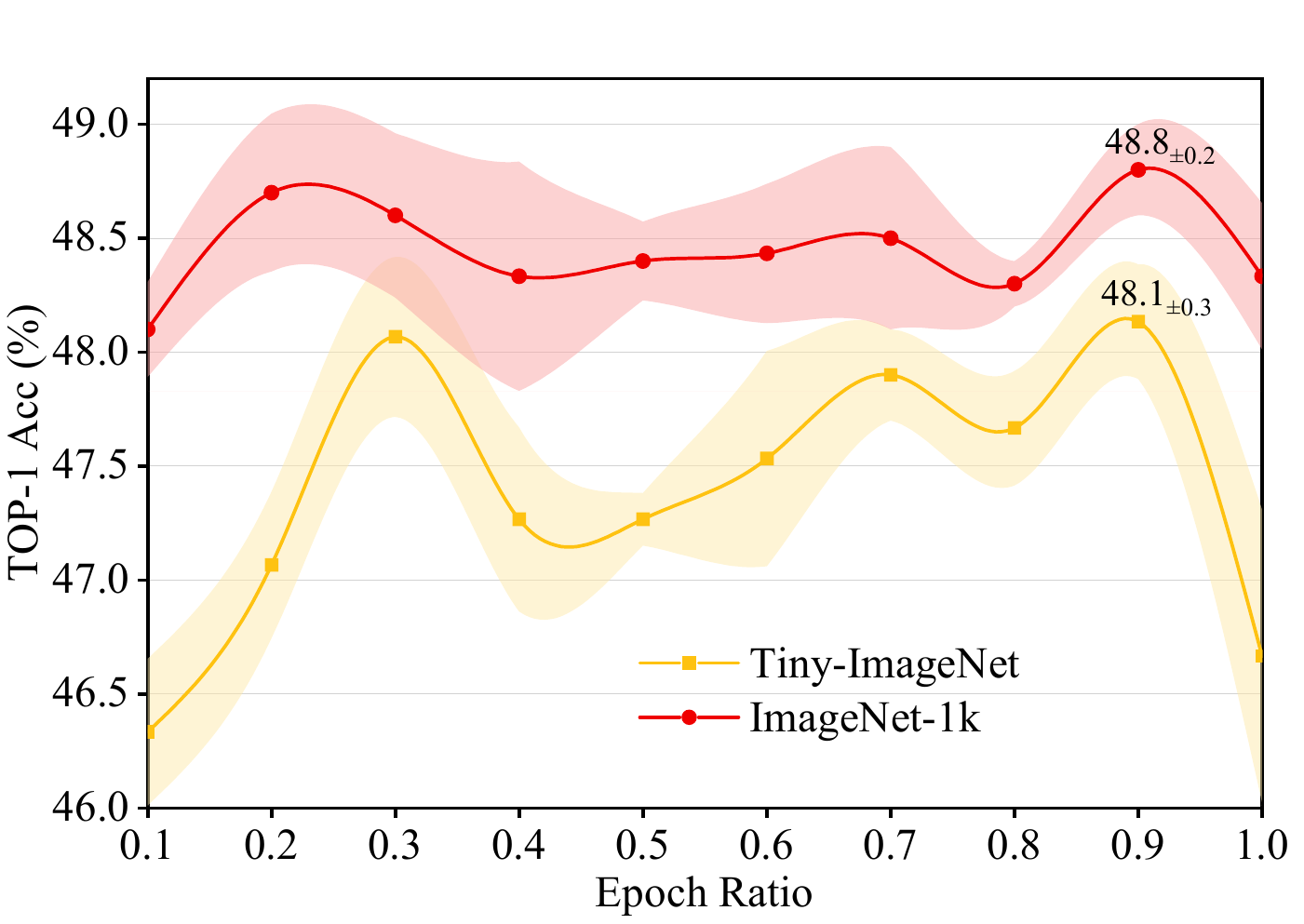}

\caption{Top-1 accuracy under different epoch ratios on Tiny-ImageNet and ImageNet-1K with $\mathrm{IPC}=10$ and teacher=student=ResNet18.}
\label{fig:epoch_ratio_tiny_imagenet_imagenet1k}
\end{figure}

\subsection{More Ablation Study}
\paragraph{Optimal Epoch Ratios for other Models.}
To determine the optimal epoch ratio $r_{\mathrm{fwd}}$ for the other teacher models (DenseNet121, MobileNetV2, and ShuffleNetV2) in DSM+, we conduct an ablation study on ImageNette under the same-architecture protocol with $\mathrm{IPC}=10$. As shown in Table~\ref{tab:epoch_ratio_models}, the optimal $r_{\mathrm{fwd}}$ values for DenseNet121, MobileNetV2, and ShuffleNetV2 are $1.0$ ($69.1\%$), $0.7$ ($54.2\%$), and $0.5$ ($50.3\%$), respectively. Overall, the results suggest that models with more parameters may require more forward passes to better capture the distribution of the current difficulty group. Since the performance gaps between the optimal $r_{\mathrm{fwd}}$ and $r_{\mathrm{fwd}}=1.0$ are small for MobileNetV2 and ShuffleNetV2, we apply the optimal $r_{\mathrm{fwd}}$ determined for ResNet18 to the other teacher models in DSM+ on other datasets and IPC settings for simplicity.

\paragraph{Epoch Ratio $r_{\mathrm{fwd}}$ on Tiny-ImageNet and ImageNet-1K.}
Due to limited computational resources, we do not further investigate the effect of the ISS screening methods on Tiny-ImageNet and ImageNet-1K. Instead, we only examine the effect of different $r_{\mathrm{fwd}}$ under $\mathrm{IPC}=10$, using ResNet18 as the teacher and student. As shown in Figure~\ref{fig:epoch_ratio_tiny_imagenet_imagenet1k}, consistent with the observations in Figure~\ref{fig:EpochRatio_SelectMethod}, larger $r_{\mathrm{fwd}}$ allows the statistics to better characterize the distribution of the current difficulty group. Both Tiny-ImageNet and ImageNet-1K achieve the highest accuracy at $r_{\mathrm{fwd}}=0.9$, reaching $48.1\%$ and $48.8\%$, respectively. Therefore, we set $r_{\mathrm{fwd}}=0.9$ for the other experiments on Tiny-ImageNet and ImageNet-1K.

\paragraph{Patch Grids in the Initial Image.}
Since DSM modifies the sample initialization method, we further investigate the effect of different patch grids on performance. We use ResNet18 as the teacher and student with $\mathrm{IPC}=10$, and the specific patch grid settings are illustrated in Figure~\ref{fig:patch_grids}. As shown in Table~\ref{tab:patch_grid}, the $1\times1$ patch grid achieves the best performance on all datasets, consistent with the findings of~\citep{FADRM}. When the patch grid increases to $2\times2$, compressing local image regions may cause information loss, which is more evident on high-resolution datasets. Therefore, we use the $1\times1$ initialization setting in all experiments.

\begin{table}[!htbp]
\centering

\caption{Top-1 accuracy under different patch grid settings with $\mathrm{IPC}=10$, using ResNet18 as both the teacher and student.}
\label{tab:patch_grid}

\small
\setlength{\tabcolsep}{5pt}
\renewcommand{\arraystretch}{1.5}

\resizebox{\linewidth}{!}{%
\begin{tabular}{@{}ccccc@{}}
\toprule
Patch Grid & CIFAR-10 & CIFAR-100 & ImageNette & ImageWoof \\
\midrule

$1\times1$ & \tbscore{54.5}{1.3} & \tbscore{66.7}{0.4} & \tbscore{67.1}{0.2} & \tbscore{44.5}{0.5} \\

$2\times2$ & \tbscore{54.3}{1.9} & \tbscore{65.0}{0.4} & \tbscore{65.1}{0.9} & \tbscore{40.2}{1.9} \\

\bottomrule
\end{tabular}%
}
\end{table}

\begin{figure}[!htbp]
\centering

\includegraphics[height=2.5cm,keepaspectratio]
{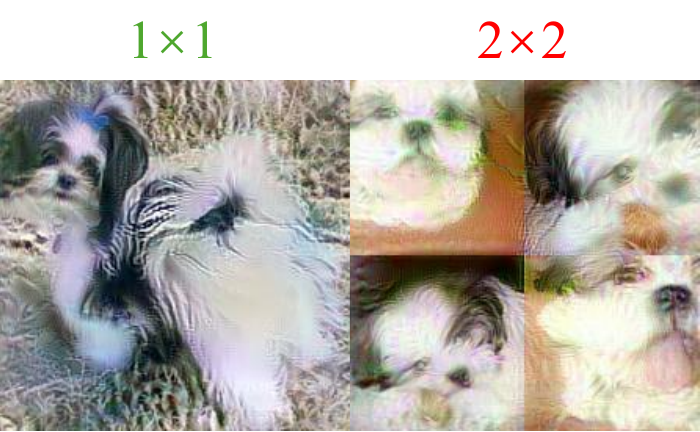}

\caption{Visualization of different patch grids.}
\label{fig:patch_grids}

\end{figure}

\paragraph{Using Images Selected by ISS as Distilled Samples.}
Recent studies~\citep{qin2024label,li2026ranking,zhong2026rd3} construct highly informative distilled datasets by selecting original samples. Inspired by these works, we evaluate using samples selected by ISS directly as distilled data, adopting the optimal screening settings in Figure~\ref{fig:EpochRatio_SelectMethod}. As shown in Table~\ref{tab:iss_only}, ISS alone underperforms DSM on all datasets, indicating that the statistical supervision provided by SUA further improves distilled data quality. Nevertheless, ISS alone achieves reasonable accuracy, demonstrating its ability to select valuable training samples.

\begin{table}[!tp]
\centering
\caption{Top-1 accuracy of ISS-selected samples and DSM with $\mathrm{IPC}=10$ and teacher=student=ResNet18.}
\label{tab:iss_only}
\begin{tabular}{@{}cccc@{}}
\toprule
Dataset & Screening Method & ISS & DSM \\
\midrule
CIFAR-10   & Back  & $43.3_{\pm 0.5}$ & $54.5_{\pm 1.3}$ \\
CIFAR-100  & Back  & $28.1_{\pm 1.1}$ & $66.7_{\pm 0.4}$ \\
ImageNette & Front & $61.1_{\pm 0.8}$ & $67.1_{\pm 0.2}$ \\
ImageWoof  & Back  & $35.4_{\pm 0.2}$ & $44.5_{\pm 0.5}$ \\
\bottomrule
\end{tabular}
\end{table}

\section{Implementation Details}
\label{implementation details}
All experiments were run on a single NVIDIA A100 GPU using PyTorch 2.12.0, CUDA 13.0.3 and Python 3.10.
\subsection{Pretraining.}

Table~\ref{tab:hyperparameter_settings} reports the pretraining hyperparameters for ShuffleNetV2, ResNet18, ResNet50, MobileNetV2 and DenseNet121 across datasets. We prioritize the configurations provided by FADRM. However, some are unavailable or fail to reproduce the reported accuracy in our environment. We therefore tune the hyperparameters to approach the reported results. To ensure equal numbers of forward passes across teachers in DSM+ under the same $r_{\mathrm{fwd}}$, we use identical pretraining epochs and configurations for all models on each dataset. For ImageNet-1K, we use the official pretrained weights from PyTorch.

After pretraining the teacher models, we compute GDS. The batch size is kept the same as in pretraining, with $\beta=1.0$ and $\epsilon=1\times10^{-6}$. The detailed procedure is shown in Algorithm~\ref{alg:gps}.

\begin{table*}[!tp]
\centering
\caption{Hyperparameter settings during pretraining and GDS.}
\label{tab:hyperparameter_settings}

\small
\setlength{\tabcolsep}{5.8pt}
\renewcommand{\arraystretch}{1.15}

\begin{tabular}{@{}ccccccc@{}}
\toprule
Hyperparameter & CIFAR-10 & CIFAR-100 & Tiny-ImageNet & ImageNette & ImageWoof & ImageNet-1K \\
\midrule
learning rate & $1.00\mathrm{E}{-03}$ & $1.00\mathrm{E}{-03}$ & $1.00\mathrm{E}{-02}$ & $1.00\mathrm{E}{-02}$ & $1.00\mathrm{E}{-02}$ & \multirow{3}{*}{\rotatebox[origin=c]{90}{PyTorch}} \\
optimizer & Adam & Adam & SGD & SGD & SGD & \\
scheduler & cos & cos & cos & cos & cos & \\
epoch & 150 & 100 & 50 & 150 & 200 & 90 \\
batch & 512 & 512 & 128 & 64 & 64 & 64 \\
$\beta$ & 1 & 1 & 1 & 1 & 1 & 1 \\
$\epsilon$ & $1.00\mathrm{E}{-06}$ & $1.00\mathrm{E}{-06}$ & $1.00\mathrm{E}{-06}$ & $1.00\mathrm{E}{-06}$ & $1.00\mathrm{E}{-06}$ & $1.00\mathrm{E}{-06}$ \\
\bottomrule
\end{tabular}
\end{table*}

\begin{algorithm}[!htbp]
\caption{Global Difficulty Score (GDS)}
\label{alg:gps}
\begin{algorithmic}[1]
\Require Original set $\mathcal{T}=\{(x_i,y_i)\}_{i=1}^{N}$; pretrained teachers $\{f_{\theta_k}\}_{k=1}^{K}$; repeated evaluations $R$ (default 5)

\For{each teacher $k=1,\ldots,K$}
    \For{$r=1,\ldots,R$}
        \For{each sample $x_i\in\mathcal{T}$}
            \State Forward $x_i$ through $f_{\theta_k}$ and compute its statistics at all BN layers
            \State Compute $s_{i,k}^{(r)}$ using the statistical distances in~\eqref{eq:gps}
        \EndFor
    \EndFor
    \State $s_{i,k}\gets \frac{1}{R}\sum_{r=1}^{R}s_{i,k}^{(r)}$
    \State Rank $s_{i,k}$ within each class: $r_{i,k}\gets\operatorname{rank}_{c}(s_{i,k})$
\EndFor

\State $\operatorname{GDS}(x_i)\gets\frac{1}{K}\sum_{k=1}^{K}r_{i,k}$ for each $x_i\in\mathcal{T}$
\State \Return Class-wise GDS rankings from easy to hard

\end{algorithmic}
\end{algorithm}

\subsection{Distillation.}
For DSM+, to simplify the experimental setup, we use ShuffleNetV2, ResNet18, MobileNetV2, and DenseNet121 as teacher models on all datasets, while using the same teacher model settings as FADRM+. During the forward passes of SUA, the batch size fed into the teacher models is kept the same as that used during pretraining to reduce statistical bias. The remaining hyperparameters are shown in Table~\ref{tab:psm_hyperparameters}, and the detailed procedure is provided in Algorithm~\ref{alg:psm_distillation}. It is worth noting that when $\mathrm{IPC}=1$, the entire dataset forms a single difficulty group. Thus, SUA cannot provide differentiated supervision across groups, while ISS only introduces difficulty variation through sample initialization. Consequently, as shown in Table~\ref{tab:main_results}, DSM(+) and FADRM(+) achieve similar performance on most datasets at $\mathrm{IPC}=1$.

For a fair comparison, we reproduced FADRM(+) using the hyperparameters reported in the original FADRM paper. DSM(+) uses FADRM(+) as its baseline and inherits the corresponding baseline settings for all shared hyperparameters. To simplify the experiments in Tables~\ref{tab:main_results} and~\ref{tab:other_sota_results}, we conducted DSM hyperparameter ablations only with ResNet18 as the teacher and $\mathrm{IPC}=10$, as shown in Figures~\ref{fig:EpochRatio_SelectMethod} and~\ref{fig:epoch_ratio_tiny_imagenet_imagenet1k}. We reused the hyperparameter settings selected under this configuration for all other teacher models and IPC values.

\begin{table*}[!tp]
\centering
\caption{Hyperparameter settings during distillation.}
\label{tab:psm_hyperparameters}

\small
\setlength{\tabcolsep}{4pt}
\renewcommand{\arraystretch}{1.15}

\begin{tabular}{@{}ccccccc@{}}
\toprule
Hyperparameter & CIFAR-10 & CIFAR-100 & Tiny-ImageNet & ImageNette & ImageWoof & ImageNet-1K \\
\midrule
Epoch ratio & 0.9 & 0.9 & 0.9 & 1.0 & 0.4 & 0.9 \\
Screening method & Back & Back & Back & Front & Back & Back \\
DSM+ teachers & \multicolumn{6}{c}{ShuffleNetV2,  MobileNetV2,  DenseNet121,  ResNet18} \\
\bottomrule
\end{tabular}
\end{table*}

\begin{algorithm}[!htbp]
\caption{Difficulty-aware Statistical Matching (DSM)}
\label{alg:psm_distillation}
\begin{algorithmic}[1]
\Require Difficulty groups $\{\mathcal{T}_{g,c}\}$ obtained by GDS; pretrained teachers $\{f_{\theta_k}\}_{k=1}^{K}$; forward ratio $r_{\mathrm{fwd}}$; pretraining epochs $\{E_k^{\mathrm{pre}}\}_{k=1}^{K}$; optimization iterations $T$
\Ensure Distilled dataset $\widetilde{\mathcal{T}}$

\State Initialize $\widetilde{\mathcal{T}}\gets\varnothing$

\For{$g=1,\ldots,\mathrm{IPC}$}

    \State \textbf{SUA:} construct difficulty-specific statistical supervision
    \For{each teacher $k=1,\ldots,K$}
        \State Freeze $\theta_k$ and initialize BN statistics with $\boldsymbol{\tau}_{k}^{\mathrm G}$
        \State Forward $\mathcal{T}_{g}=\bigcup_c\mathcal{T}_{g,c}$ through $f_{\theta_k}$ for $\left\lfloor r_{\mathrm{fwd}}E_k^{\mathrm{pre}}\right\rfloor$ epochs
        \State Update only BN running statistics according to~\eqref{eq:bn_running_statistics}
        \State Store the updated statistics as $\boldsymbol{\tau}_{k,g}^{\mathrm{SUA}}$
    \EndFor

    \State \textbf{ISS:} initialize distilled data with matched difficulty
    \For{each class $c$}
        \State Select real samples from $\mathcal{T}_{g,c}$ using ISS
        \State Initialize $\widetilde{x}_{g,c}^{(0)}$ with the selected samples
    \EndFor

    \State \textbf{Distillation:}
    \For{$t=1,\ldots,T$}
        \State Select teacher $k$ (before SUA) and use $\boldsymbol{\tau}_{k,g}^{\mathrm{SUA}}$ as statistical target
        \State Update $\widetilde{\mathcal{T}}_g$ by minimizing~\eqref{eq:decoupled_distillation} with the SUA statistics
    \EndFor

    \State $\widetilde{\mathcal{T}}\gets\widetilde{\mathcal{T}}\cup\widetilde{\mathcal{T}}_g$

\EndFor

\State \Return $\widetilde{\mathcal{T}}$

\end{algorithmic}
\end{algorithm}

\subsection{Soft Label Generation.}
To maintain consistency in the soft label distributions, we generate all soft labels using the pretrained teachers before SUA updates. To simplify the experimental setup and ensure a fair comparison, we assign equal weights to all teachers when generating soft labels for both FADRM+ and DSM+, and use the remaining hyperparameters reported in the original FADRM paper.

\subsection{Post-evaluation.}
For a fair comparison, we use the hyperparameters reported in the original FADRM paper for each student model during post-evaluation. We evaluate FADRM(+) and DSM(+) three times each and report the mean and standard deviation of top-1 accuracy.

\subsection{Computation of the Difficulty Score.}
Figures~\ref{fig:difficulty_ImageNette_ImageWoof}, \ref{fig:difficulty_CIFAR-10_CIFAR-100} and \ref{fig:difficulty_Tiny-ImageNet_ImageNet-1k} use GDS to measure sample difficulty, while Figure~\ref{fig:EL2N_difficulty_ImageNette_ImageWoof} uses the cross-entropy loss used in EL2N~\citep{EL2N}. For the original data, we compute the difficulty score of each sample, sort the samples within each class by increasing score, and plot the scores against their ranks to obtain a difficulty curve for each class. The envelope of these curves forms the yellow shaded region. For clarity, we omit the rank ticks for the original data. For FADRM and DSM, we compute the difficulty score of each distilled sample and plot the mean score across classes at each distillation batch index $id_{\mathrm{IPC}}$ as a point.

\section{More Visualization}
To more intuitively illustrate the changes in the difficulty of distilled samples, we visualize the samples generated by FADRM and DSM with ResNet18 as the teacher model and $\mathrm{IPC}=10$. Due to space limitations, we only show results on CIFAR-10, ImageNette and ImageWoof.

As shown in Figure~\ref{fig:FADRM_cifar10}, ~\ref{fig:FADRM_imagenette} and~\ref{fig:FADRM_imagewoof}, the distilled samples generated by FADRM are typically concentrated within a relatively narrow difficulty range. In contrast, as shown in Figure~\ref{fig:PSM_cifar10}, ~\ref{fig:PSM_imagenette} and~\ref{fig:PSM_imagewoof}, DSM generates distilled samples that cover a broader difficulty range, promoting the student model to learn a more complete difficulty structure.

\begin{figure}[!htbp]
\centering
\includegraphics[width=\linewidth]{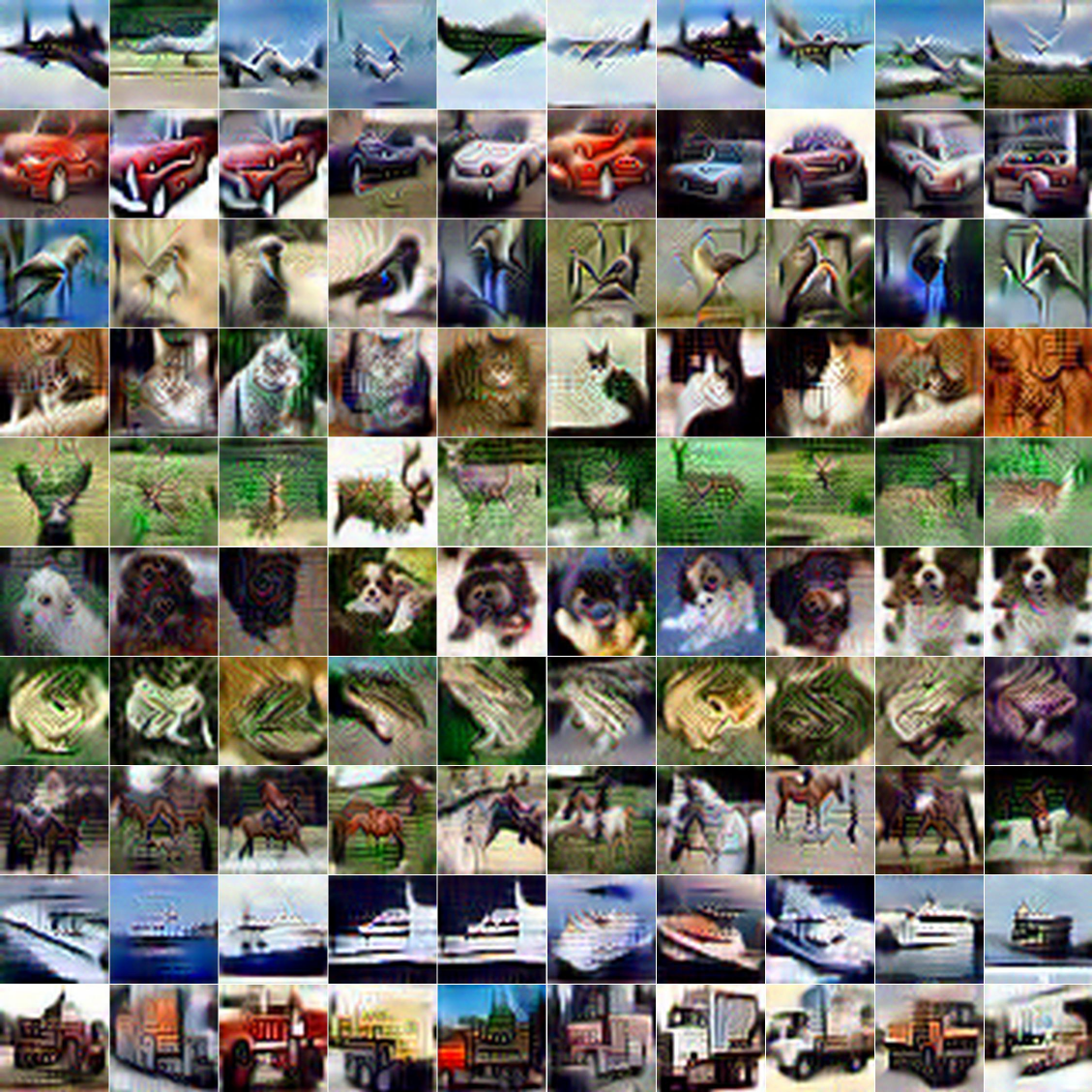}
\caption{Visualization of the distilled CIFAR-10 samples generated by FADRM with $\mathrm{IPC}=10$ and teacher=ResNet18.}
\label{fig:FADRM_cifar10}
\end{figure}

\begin{figure}[!htbp]
\centering
\includegraphics[width=\linewidth]{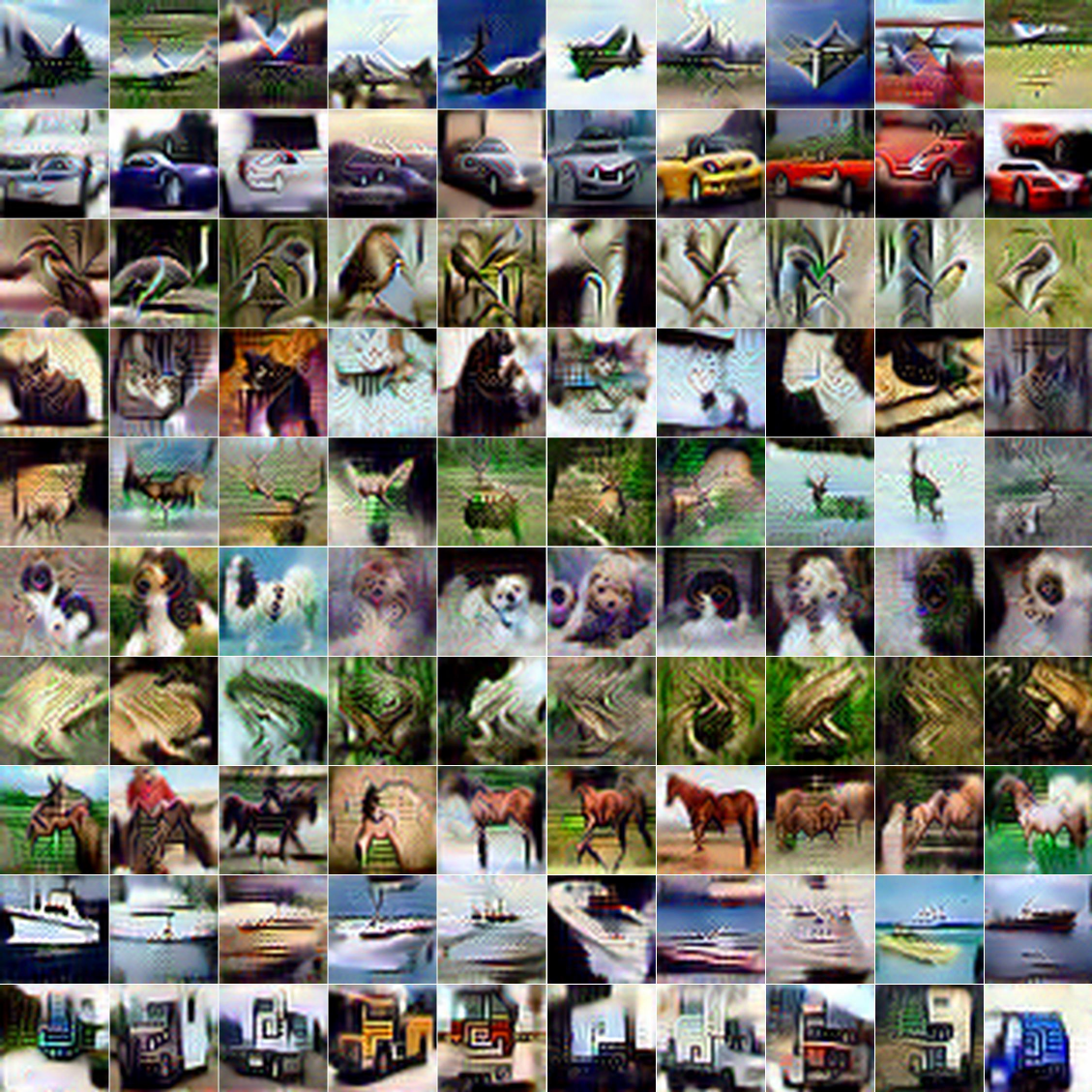}
\caption{Visualization of the distilled CIFAR-10 samples generated by DSM with $\mathrm{IPC}=10$ and teacher=ResNet18.}
\label{fig:PSM_cifar10}
\end{figure}

\begin{figure}[!htbp]
\centering
\includegraphics[width=\linewidth]{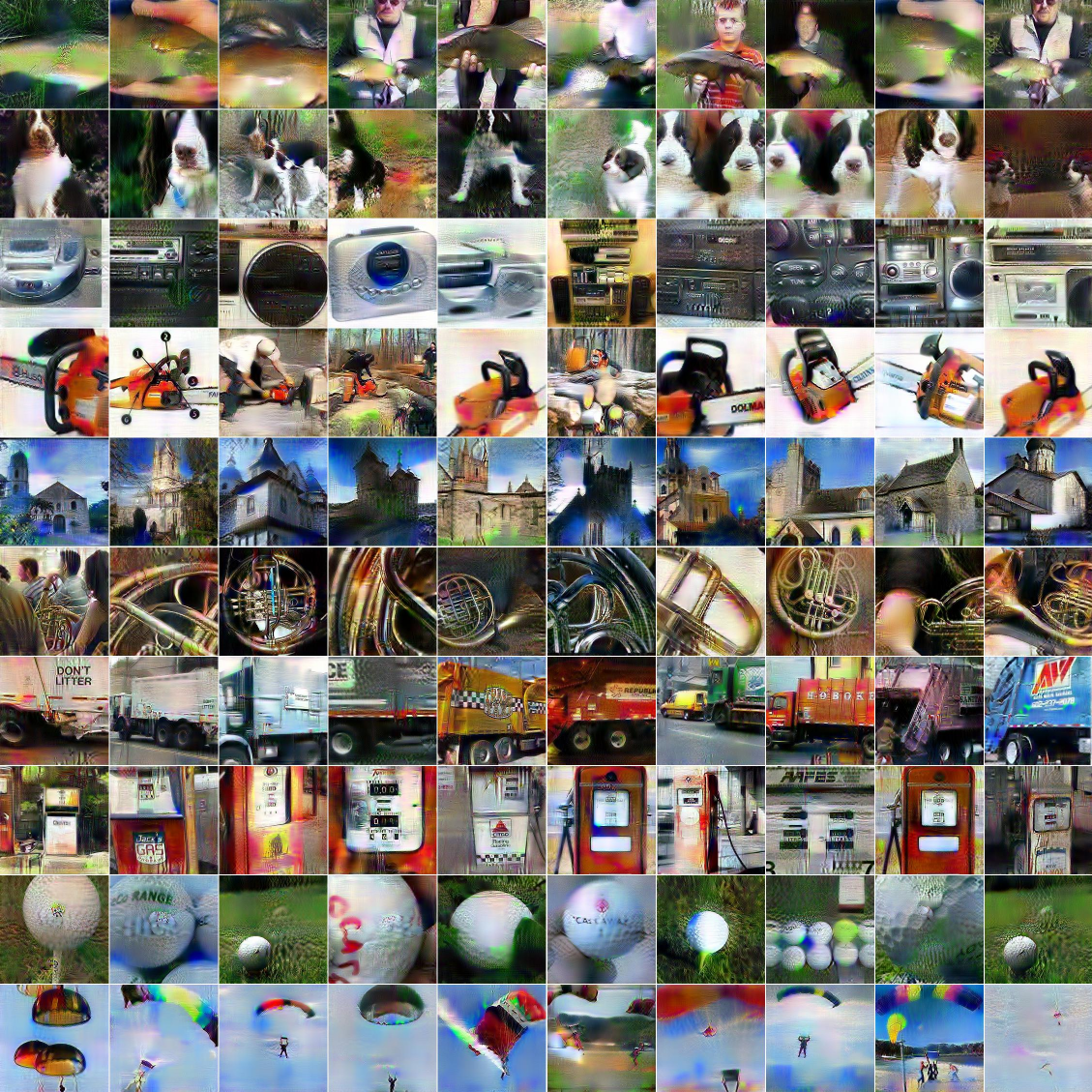}
\caption{Visualization of the distilled ImageNette samples generated by FADRM with $\mathrm{IPC}=10$ and teacher=ResNet18.}
\label{fig:FADRM_imagenette}
\end{figure}

\begin{figure}[!htbp]
\centering
\includegraphics[width=\linewidth]{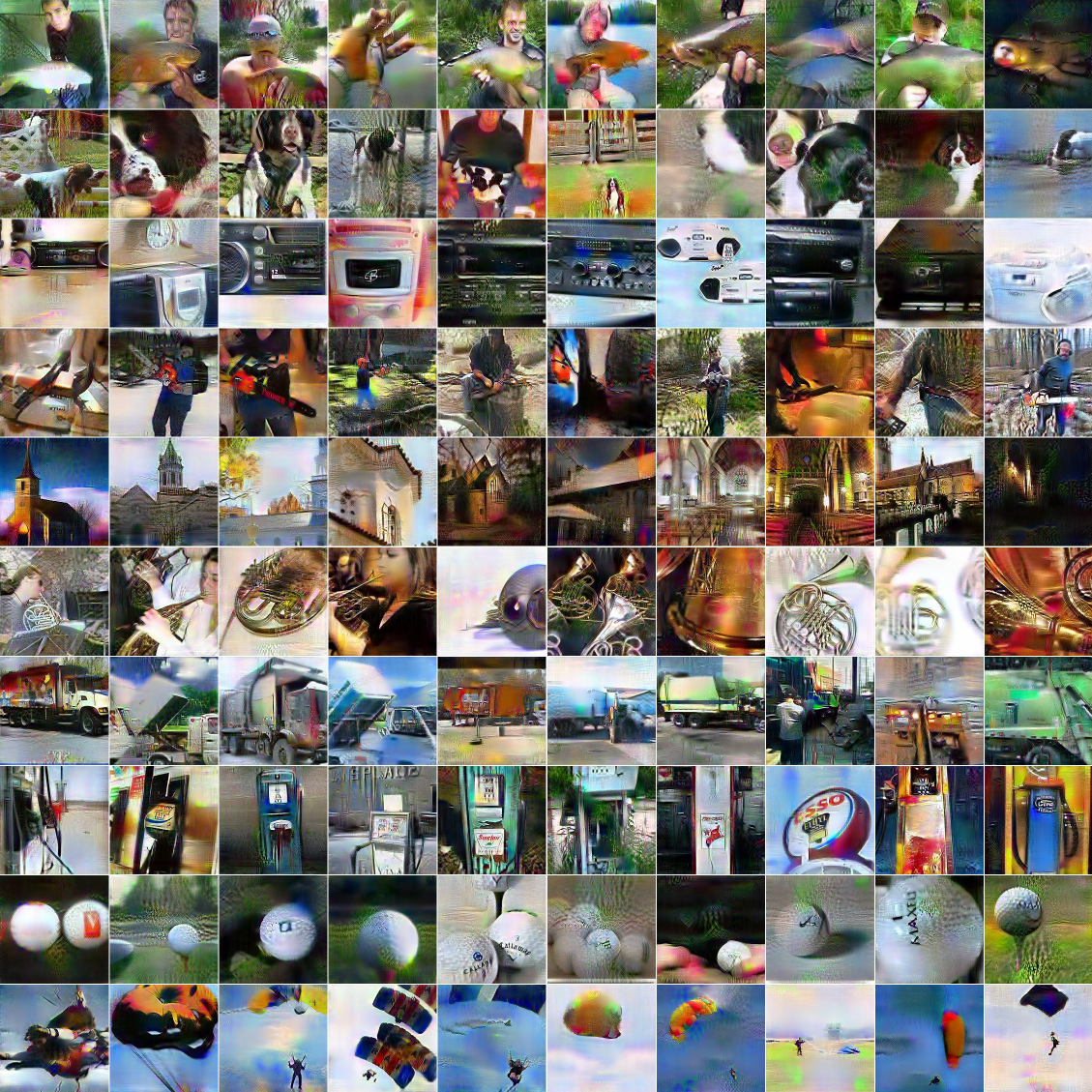}
\caption{Visualization of the distilled ImageNette samples generated by DSM with $\mathrm{IPC}=10$ and teacher=ResNet18.}
\label{fig:PSM_imagenette}
\end{figure}

\begin{figure}[!htbp]
\centering
\includegraphics[width=\linewidth]{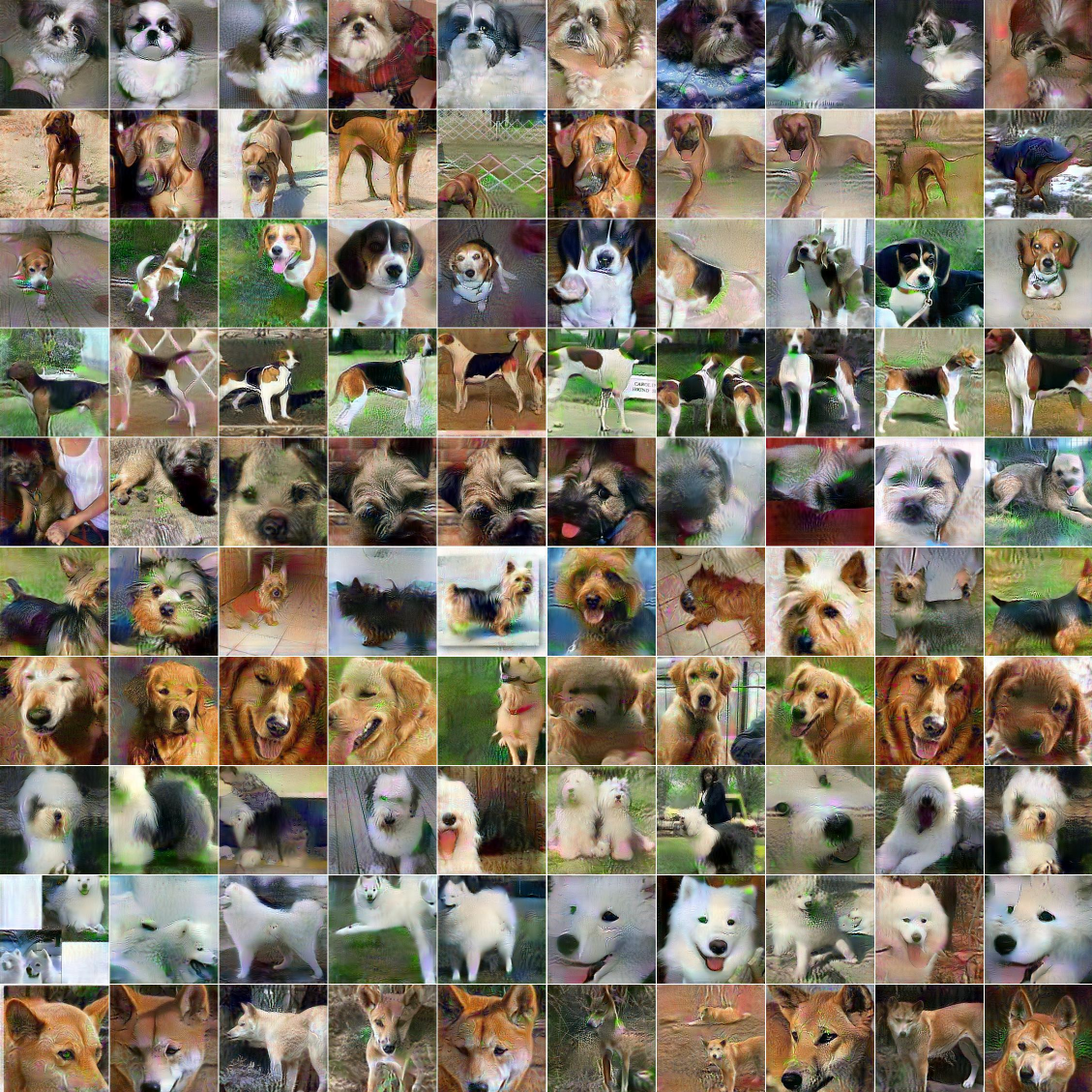}
\caption{Visualization of the distilled ImageWoof samples generated by FADRM with $\mathrm{IPC}=10$ and teacher=ResNet18.}
\label{fig:FADRM_imagewoof}
\end{figure}

\begin{figure}[!htbp]
\centering
\includegraphics[width=\linewidth]{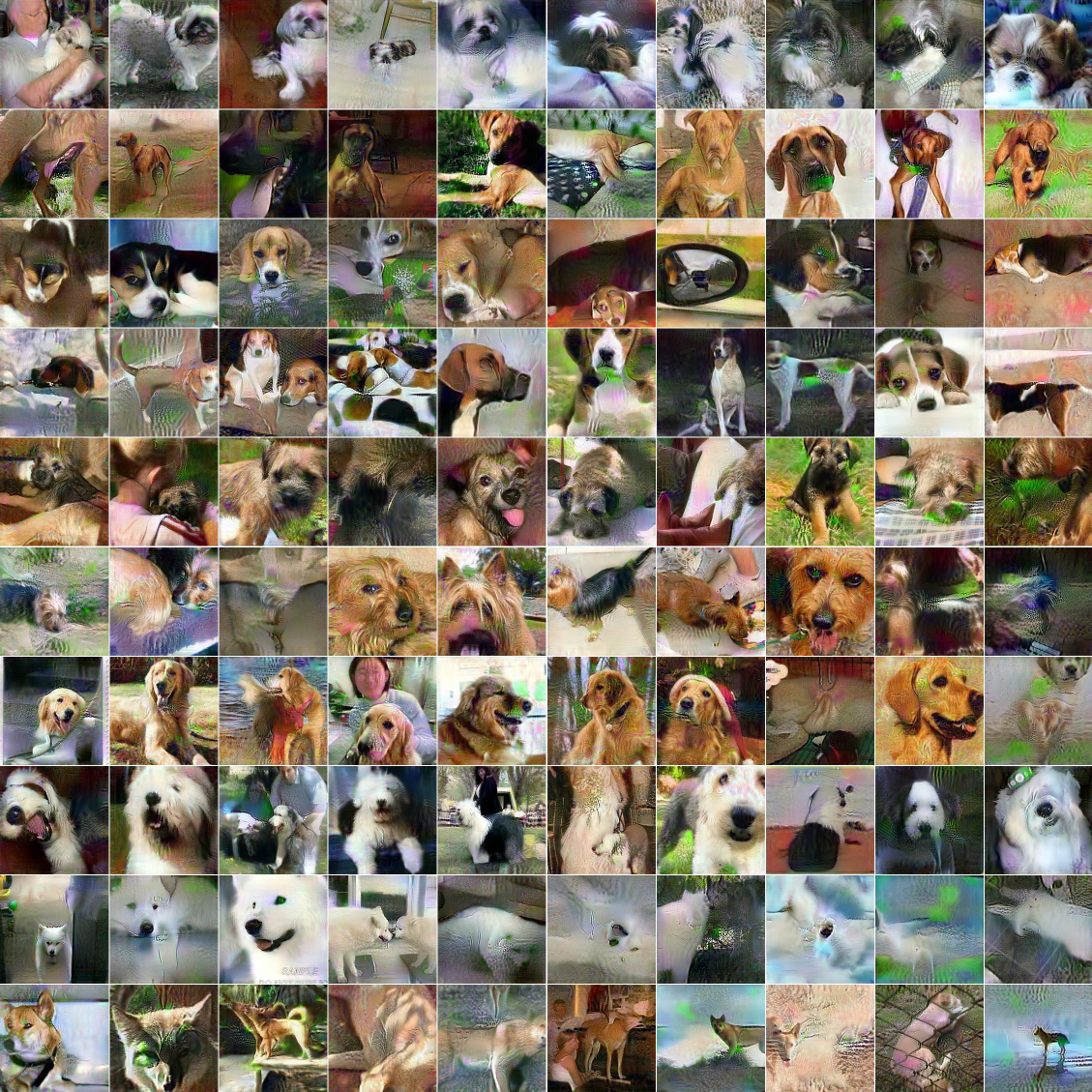}
\caption{Visualization of the distilled ImageWoof samples generated by DSM with $\mathrm{IPC}=10$ and teacher=ResNet18.}
\label{fig:PSM_imagewoof}
\end{figure}

\FloatBarrier

\end{document}